\documentclass[10pt,journal,compsoc]{IEEEtran}
\usepackage{amsmath,amsfonts}
\usepackage{amsthm,amssymb}
\usepackage{amsmath,amssymb,amsfonts,dsfont,pifont,bm,bbm,mathrsfs,mathtools,nicefrac}
\usepackage{xcolor,algorithm,listings}
\usepackage{algorithmicx}
\usepackage{algpseudocode}
\usepackage{array}
\usepackage[caption=false,font=normalsize,labelfont=sf,textfont=sf]{subfig}
\usepackage{textcomp}
\usepackage{booktabs,multirow,adjustbox,diagbox,threeparttable,multicol,colortbl,tabularray}
\usepackage{stfloats}
\usepackage{url,setspace}
\usepackage{verbatim}
\usepackage{graphicx}
\usepackage{xspace}
\usepackage{cite}
\def\BibTeX{{\rm B\kern-.05em{\sc i\kern-.025em b}\kern-.08em
    T\kern-.1667em\lower.7ex\hbox{E}\kern-.125emX}}
\newtheorem{theorem}{Theorem}
\newtheorem{lemma}{Lemma}
\usepackage[pagebackref,breaklinks,colorlinks,linkcolor=blue,citecolor=blue, bookmarks=false]{hyperref}
\usepackage[capitalize]{cleveref}  
\Crefname{section}{Section}{Sections}
\Crefname{table}{Table}{Tables}
\Crefname{figure}{Figure}{Figures}
\Crefname{equation}{Equation}{Equations}
\definecolor{codegreen}{rgb}{0,0.6,0}
\definecolor{codegray}{rgb}{0.5,0.5,0.5}
\definecolor{codepurple}{rgb}{0.58,0,0.82}
\definecolor{backcolour}{rgb}{1.0,1.0,1.0}
\lstdefinestyle{mystyle}{
    backgroundcolor=\color{backcolour},
    commentstyle=\color{codegreen},
    keywordstyle=\color{magenta},
    numberstyle=\tiny\color{codegray},
    stringstyle=\color{codepurple},
    basicstyle=\ttfamily\scriptsize,
    breakatwhitespace=false,
    breaklines=true,
    captionpos=b,
    keepspaces=true,
    numbers=left,
    numbersep=5pt,
    showspaces=false,
    showstringspaces=false,
    showtabs=false,
    tabsize=2
}
\newcommand{\balpha}[1]{\bar{\alpha}_{#1}}

\newcommand{\revise}[1]{\textcolor{black}{#1}}
\newcommand{\newrevise}[1]{\textcolor{black}{#1}}
\newcommand{\method}{DMT\xspace}
\newcommand{\supp}{\textit{Supplementary Material}\xspace}
\newcommand{\sqq}[1]{\textcolor{black}{#1}}
\newcommand{\yr}[1]{\textcolor{black}{#1}}

\begin{document}

\title{A Diffusion Model Translator for Efficient Image-to-Image Translation}

\author{
Mengfei Xia, Yu Zhou, Ran Yi, Yong-Jin Liu,~\IEEEmembership{Senior Member,~IEEE}, Wenping Wang,~\IEEEmembership{Fellow,~IEEE}
\IEEEcompsocitemizethanks{\IEEEcompsocthanksitem  Mengfei Xia and Yu Zhou are with the MOE-Key Laboratory of Pervasive Computing, Department of Computer Science and Technology, Tsinghua University, Beijing 100084, China (email: \{xmf20, yzhou20\}@mails.tsinghua.edu.cn).
\IEEEcompsocthanksitem Ran Yi is with the Department of Computer Science and Engineering, Shanghai Jiao Tong University, Shanghai, China (email: ranyi@sjtu.edu.cn).
\IEEEcompsocthanksitem Yong-Jin Liu is with the MOE-Key Laboratory of Pervasive Computing, Department of Computer Science and Technology, Tsinghua University, Beijing 100084, China (email: liuyongjin@tsinghua.edu.cn).
\IEEEcompsocthanksitem Wenping Wang is with the Department of Computer Science and Computer Engineering at Texas A\&M University, Texas, The United States (email: wenping@tamu.edu).
\IEEEcompsocthanksitem Corresponding authors: Ran Yi and Yong-Jin Liu.}
}

\markboth{IEEE TRANSACTIONS ON PATTERN ANALYSIS AND MACHINE INTELLIGENCE, VOL. 46, NO. 12, DECEMBER 2024}%
{Xia \MakeLowercase{\textit{et al.}}: A Diffusion Model Translator for Efficient Image-to-Image Translation}


\IEEEtitleabstractindextext{
\begin{abstract}

Applying diffusion models to image-to-image translation (I2I) has recently received increasing attention due to its practical applications.
Previous attempts \sqq{inject information} from the source image into each denoising step for an iterative refinement, \sqq{thus resulting in a} time-consuming implementation.
We propose an efficient method that equips a diffusion model with a lightweight translator, dubbed a Diffusion Model Translator (\method), to accomplish I2I. 
Specifically, we first \sqq{offer} theoretical justification that \sqq{in employing} the pioneer\sqq{ing} DDPM work \sqq{for} the I2\revise{I} task, it is \sqq{both} feasible and sufficient to transfer the distribution from one domain to another only at some intermediate step.
We \sqq{further} observe that the translation performance highly depends on the \sqq{chosen timestep for} domain transfer, and therefore propose a practical strategy to automatically select an appropriate timestep for a given task.
We evaluate our approach on a range of I2I applications, \sqq{including} image stylization, image colorization, segmentation to image, and sketch to image, to validate its efficacy and general utility.
\sqq{The c}omparisons show that our \method surpasses existing methods in both quality and efficiency.
Code will be made publicly available.

\end{abstract}

}

\maketitle



\section{Introduction}\label{sec:intro}

\IEEEPARstart{A} diffusion probabilistic model~\cite{sohl2015deep,ho2020denoising,song2020score,song2020denoising}, also known as a diffusion model, is a generative model that consists \sqq{of} (1) a forward diffusion process that gradually adds noise to a data distribution until it becomes a simple latent distribution (\textit{e.g.}, Gaussian), and (2) a reverse process that \sqq{begins with} a random sample in the latent distribution \sqq{and employs} a learned network to revert the diffusion process, \sqq{thereby generating} a data point in the original distribution. 
Among all the variants of the diffusion model, the denoising diffusion probabilistic model (DDPM)~\cite{ho2020denoising} offers the advantage of a simple training procedure by exploring an explicit connection between the diffusion model and denoising score matching.
Recent studies have demonstrated the compelling performance of DDPM in high-fidelity image synthesis~\cite{ho2020denoising, nichol2021improved, dhariwal2021diffusion}. 

Despite its rapid development, there are relatively few studies on applying the diffusion model to conditional generation, which is \sqq{a key requirement} for many real-world applications, such as the well-known image-to-image (I2I) task~\cite{isola2017image} that translate a source image of one style into another target image of \sqq{a different} style.
Unlike unconditional generation, conditional generation \sqq{necessitates constraining} synthesized result with an input sample in the source domain as the \revise{content guidance}. 
Therefore, to handle an I2I task using DDPM, existing methods~\cite{sasaki2021unit,saharia2021palette,choi2021ilvr,liu2021more,wang2022pretraining} inject the information from an input source sample into every single denoising step in the reverse process (see \cref{fig:method}a).
In this way, each denoising step explicitly relies on its previous step, making it inefficient to learn the step-wise injection.

\begin{figure*}[t]
\centering
\includegraphics[width=1.\textwidth]{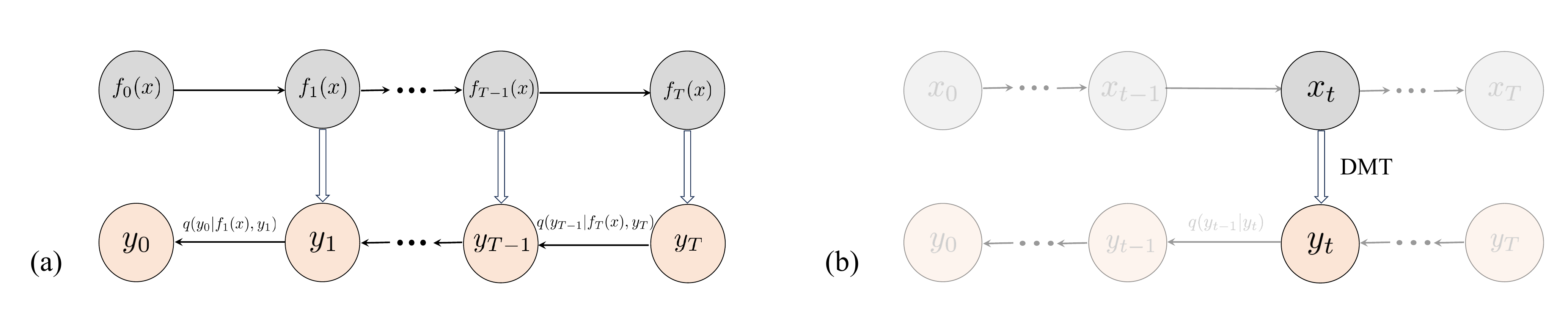}
\vspace{-15pt}
\caption{
    \textbf{Conceptual comparison} between (a) existing methods~\cite{saharia2021palette,choi2021ilvr,liu2021more} and (b) our \method.
    $\{x_t\}_{t=0}^T$ represent different states of the \revise{input from the source domain}, while $y_T \to y_0$ stands for the denoising process of DDPM.
    Here, $T$ denotes the total number of noise-adding steps in the diffusion process.
    Instead of using the \revise{information $f_t(x)$ from the source domain} (which can be the original or noisy image) for an iterative refinement at \textit{each} denoising step $t,t=0,1,\cdots,T$, \method accomplishes the I2I task efficiently by learning an efficient translation module at just one \textit{preset} timestep and fully reusing the pre-trained DDPM.
    How to select an appropriate translation timestep is discussed in \cref{subsec:Optimal}.
}
\label{fig:method}
\vspace{-10pt}
\end{figure*}

In this work\sqq{,} we investigate a more efficient approach to applying DDPM to I2I tasks by endowing a pre-trained DDPM with a translator, \sqq{which we name} {\em Diffusion Model Translator} (\method).
First, we provide a \textit{theoretical} proof that given two diffusion processes on two different image domains involved in an I2I task, it is feasible to accomplish the I2I task by shifting a distribution from one process to \sqq{another} at a particular timestep with appropriate reparameterization.
Based on this theoretical justification, we develop a new efficient DDPM pipeline, as illustrated in \cref{fig:method}b.
Assuming that a DDPM has been prepared for one image domain $y_0$, we use it to decode the latent that is shifted from another domain $x_0$.
To accomplish the domain shift, we apply the same forward diffusion process onto $x_0$ and $y_0$ until a pre-defined timestep $t$, and then employ a neural network to translate $x_t$ to $y_t$ as a typical I2I problem.

There are two major advantages \sqq{to} our approach.
First, the training of \method is independent of DDPM and can be executed very efficiently.
Second, \method can benefit from using all the previous techniques in the I2I field (\textit{e.g.}, such as Pix2Pix~\cite{isola2017image}, TSIT~\cite{jiang2020tsit}, SPADE~\cite{park2019SPADE}, and SEAN~\cite{Zhu_2020_CVPR}), for a better performance.
Furthermore, regarding the choice of the timestep $t$ to perform domain transfer, we propose a practical strategy to automatically select an appropriate timestep for a given data distribution.

To empirically validate the \revise{efficacy} of our method, we conducted evaluation on four I2I tasks: image stylization, image colorization, segmentation to image, and sketch to image.
Both qualitative and quantitative results demonstrate the superiority of our method over existing diffusion-based alternatives as well as the GAN-based counterparts of \method.

\section{Related Work}\label{sec:related-work}

\noindent\sqq{In a forward diffusion process,} a \textbf{Diffusion probabilistic model (DPM)}~\cite{sohl2015deep,ho2020denoising} transforms a given data distribution into a simple latent distribution, such as a Gaussian distribution. 
Due to its strong capabilities, DPM has achieved great success in various fields, including speech synthesis~\cite{chen2020wavegrad,kong2020diffwave}, video synthesis~\cite{ho2022video,ho2022imagen}, image super-resolution~\cite{saharia2021image,li2022srdiff}, conditional generation~\cite{choi2021ilvr,wang2022pretraining}\sqq{,} and image-to-image translation~\cite{saharia2021palette,sasaki2021unit}.
Denoising diffusion probabilistic model (DDPM)~\cite{ho2020denoising} \revise{assumes} the Markovian property of the forward diffusion process.
For a dataset of images, the forward diffusion process is realized by corrupting each image $x_0$ \sqq{through the addition of} standard Gaussian noise to reduce it into a completely random noise image. 
Formally, given the variance schedules $\alpha_t\in[0,1],t=1,2,\cdots,T, \beta_t=1-\alpha_t$, we can write the Markov chain as:
\begin{align}
& q(x_{1:T}|x_0)=\prod_{t=1}^T q(x_t|x_{t-1}),\\
& q(x_t|x_{t-1})\sim\mathcal N(x_t;\sqrt{\alpha_t}x_{t-1},\beta_t I),
\end{align}
where $x_T\sim\mathcal N(x_T;0,I)$ and $I$ is the identity matrix.

When reversing this diffusion process, DDPM serves as a generator for data generation in the form $p_{\theta}(x_0) = \int p_{\theta}(x_{0:T})dx_{1:T}$ starting from $x_T$:
\begin{align}
& p_{\theta}(x_{0:T})=p_{\theta}(x_T)\prod_{t=1}^T p_{\theta}(x_{t-1}|x_t),\\
& p_{\theta}(x_{t-1}|x_t)\sim\mathcal N(x_{t-1};\mu_{\theta}(x_t,t),\Sigma_{\theta}(x_t,t))\sqq{,}
\end{align}
so that any sample $x_T$ in the latent distribution will be mapped back to $x_0$ in the original data distribution. 
To achieve its reverse process for image synthesis, DDPM parameterizes the mean $\mu_{\theta}(x_t,t)$ by a time-dependent model $\epsilon_{\theta}(x_t,t)$ and optimizes the following simplified objective function:
%
%
%
%
\begin{align}
\mathcal{L} = \mathbb{E}_{q(x_0, t, \epsilon)}\left[\|\epsilon - \epsilon_{\theta}(\sqrt{\balpha{t}}x_0 + \sqrt{1 - \balpha{t}}\epsilon, t)\|^2 \right].
\end{align}

\begin{figure*}[t]
\centering
\includegraphics[width=1.\textwidth]{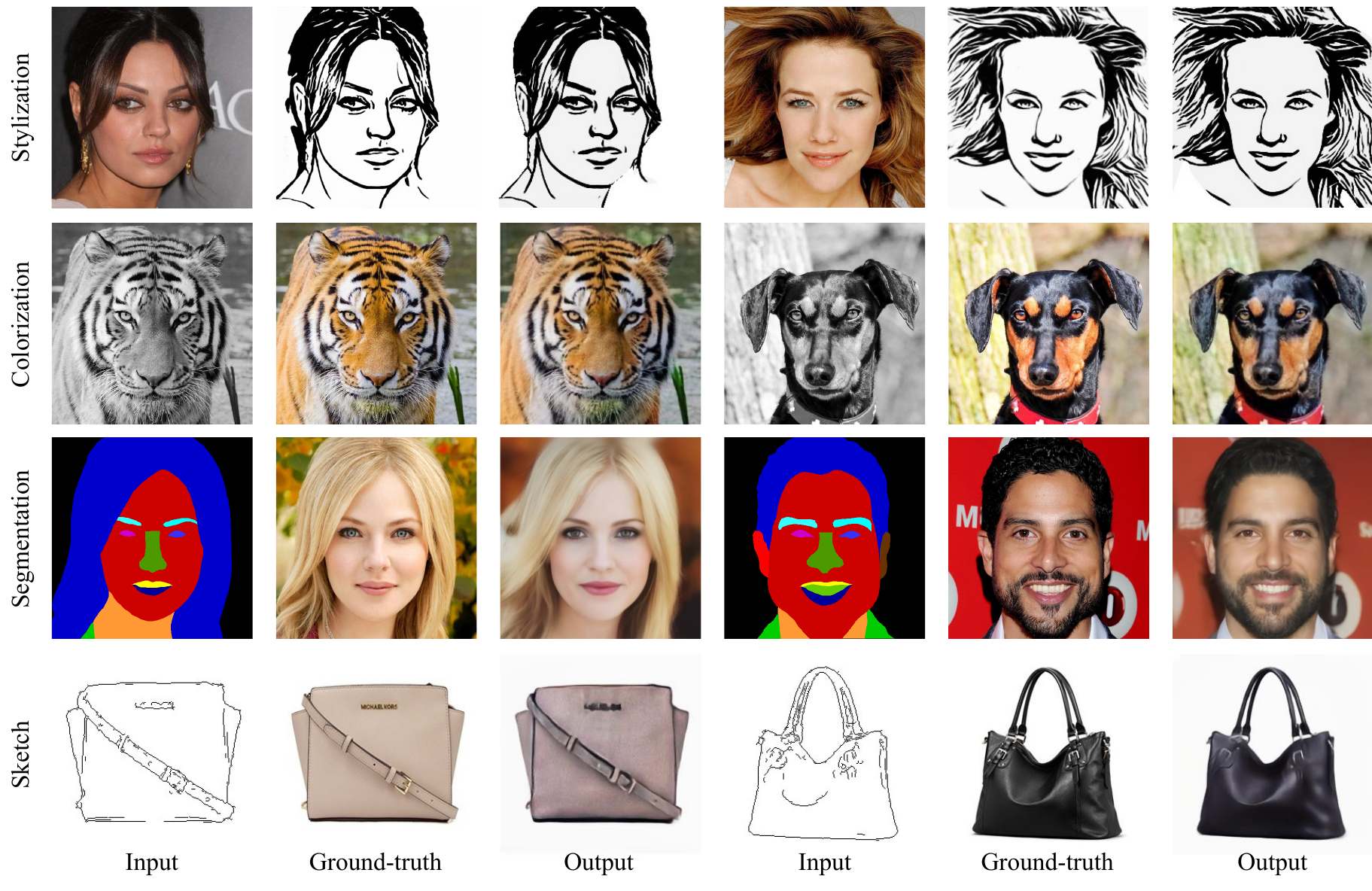}
\vspace{-10pt}
\caption{
    \textbf{Qualitative results} of our proposed \method on four I2I tasks: image stylization, image colorization, segmentation to image, and sketch to image. Here we equip a pre-trained DDPM with an efficient translation module.
    Our approach makes adequate use of the content information from the input \revise{condition} as well as the domain knowledge contained in the learned denoising process.
    %
}
\label{fig:teaser}
\vspace{-5pt}
\end{figure*}

\noindent\textbf{Faster DPM} attempts to explore shorter trajectories rather than the complete reverse process, while ensuring that the synthesis performance is comparable to the original DPM.
Some existing methods seek the trajectories using the grid search~\cite{chen2020wavegrad}.
However, this is only suitable for short reverse processes because its time complexity grows exponentially.
Other methods try to find optimal trajectories by solving a least-cost-path problem with a dynamic programming (DP) algorithm~\cite{watson2021learning,bao2022analyticdpm}.
Another representative category of fast sampling methods use\sqq{s} high-order differential equation (DE) solvers~\cite{jolicoeur2021gotta,Liu0LZ22,PopovVGSKW22,tachibana2021taylor,lu2022dpm}.
Some GAN-based methods also consider larger sampling step size\sqq{. For instance,} \cite{xiao2022DDGAN} \sqq{demonstrates learning a multi-modal distribution within a conditional GAN using a larger step size.}

\noindent\textbf{Image-to-image translation} (I2I) aims to translate an input image from a given source domain to another image in a given target domain, with input-output paired training data~\cite{isola2017image}.
To this end, the conditional generative adversarial network (cGAN) is designed to inject the information of the input image into the generation decoder with the adversarial loss~\cite{mirza2014conditional,goodfellow2014generative}.
The cGAN-based algorithms has demonstrated high quality on many I2I tasks~\cite{dong2017semantic,kaneko2017generative,karacan2016learning,ledig2017photo,sangkloy2017scribbler,wang2016generative,zhang2017age,jiang2020tsit,park2019SPADE,Zhu_2020_CVPR}.
However, due to their training instability and the severe mode collapse issue, it is hard for the cGAN-based methods to generate diverse high-resolution images.
Recently, DPM has been applied to the I2I task.
Palette~\cite{saharia2021palette} introduces the novel DPM framework to the I2I task by injecting the input \sqq{into} each sampling step for refinement.
Some methods use pre-trained image synthesis models for the I2I task~\cite{wang2022pretraining}.
Despite the high quality of synthesized images, the generation process of these existing methods is extremely time-consuming. 
Our \revise{work} tackles this issue by proposing a new DDPM method for the I2I task that works efficiently, without the time-consuming requirement of having to inject the input source information in every \sqq{denoising} step. 
\sqq{Although unpaired data are more accessible for translation tasks, the advantages of paired image-to-image (I2I) tasks, such as reduced data demands and enhanced synthesis quality, have made them a significant research focus.}

\section{Method}\label{sec:method}

\subsection{Markov process of translation mappings}\label{subsec:Markov_Process}

For an I2I task, traditional DDPM methods directly approximate the real distribution $q(y_0|x_0)$ in which $x_0, y_0$ are paired data from the source domain ${\mathcal D}_x$ and the target domain ${\mathcal D}_y$, respectively.
In contrast, we construct a translation module $p_{\theta}(y_t|x_t)$, which bridges the input \revise{condition} and the pre-trained DDPM.
\sqq{Accordingly}, we can approximate the $q(y_0|x_0)$ using the learned intermediate translation module.
Specifically, given a noise-adding schedule of the forward variance process $\beta_i\in[0,1],t=1,2,\cdots T$, $\alpha_i=1-\beta_i$ and $\balpha{t}=\prod_{i=1}^t\alpha_t$, we first generalize the forward Markov process to the joint distribution of $(x_{1:t}, y_{1:t})$ as below:
\begin{align}
q(y_{1:t},x_{1:t}|y_0, x_0)& =\prod_{i=1}^t q(x_i|x_{i-1})\prod_{j=1}^t q(y_j|y_{j-1}), \label{eq:3.1} \\
q(x_i|x_{i-1})& \sim\mathcal N(x_i;\sqrt{\alpha_i}x_{i-1},\beta_i I), \\
q(x_t|x_0)& \sim\mathcal N(x_t;\sqrt{\balpha{t}}x_0,(1-\balpha{t}) I), \\
q(y_j|y_{j-1})& \sim\mathcal N(y_j;\sqrt{\alpha_i}y_{i-1},\beta_i I), \\
q(y_t|y_0)& \sim\mathcal N(y_t;\sqrt{\balpha{t}}y_0,(1-\balpha{t}) I).
\end{align}

The corresponding DDPM trained on the target domain provides a reverse Markov process to approximate $q(y_0)$ from a sample $y_T$ \sqq{drawn} from the standard Gaussian distribution, \newrevise{\textit{i.e.},} $y_T\sim\mathcal N(y_T;0,I)$.
Note that during the denoising process, $y_i$ is only determined by $y_{i+1}$ and irrelevant to $x_{0:t}$ for $i\in[0,t-1]$. We choose to construct the translation mapping at some specified step\footnote{The selection of this specified step is discussed in Section \ref{sec:exp}.} of the diffusion forward process using $p_{\theta}(y_t|x_t)$, which induces the following Markov process:
%

\begin{align}
p_{\theta}(y_{0:t},x_{1:t}|x_0)=p_{\theta}(y_t|x_t)\prod_{i=1}^t q(x_i|x_{i-1})\prod_{j=1}^t q(y_{j-1}|y_j), \label{eq:3.2}
\end{align}
where $q(y_{j-1}|y_j)$ is the denoising process of the pre-trained DDPM. 

\subsection{Translation mappings of DDPM}\label{subsec:Transfer_Mappings}

\sqq{Let} $p_{\theta}(y_0|x_0)=\int p_{\theta}(y_{0:t},x_{1:t}|x_0)dy_{1:t}dx_{1:t}$ \sqq{represent} the sampling distribution of $q(y_0|x_0)$\sqq{, where} $p_{\theta}(y_t|x_t)$ \sqq{serves to bridge} the two domains.
By making use of the variational lower bound to optimize the negative log-likelihood, we have the following lemma:
\begin{lemma}\label{lem:1}
The negative log-likelihood of $-\log p_{\theta}(y_0|x_0)$ has the following upper bound,
\begin{align}
-\log p_{\theta}(y_0|x_0)\leqslant\mathbb E_q\left[\log\frac{q(y_{1:t},x_{1:t}|y_0, x_0)}{p_{\theta}(y_{0:t},x_{1:t}|x_0)}\right],
\end{align}
where $q=q(y_{1:t},x_{1:t}|y_0, x_0).$
\end{lemma}
In other words, the translation mapping can be learned by optimizing the variational lower bound:
\begin{align}\label{eq:3.3}
\mathcal L_{CE}&=\mathbb -\mathbb E_{q(y_0|x_0)}\left[\log p_{\theta}(y_0|x_0)\right]\\
&\leqslant\mathbb E_{q(y_{0:t},x_{1:t}|x_0)}\left[\log\frac{q(y_{1:t},x_{1:t}|y_0, x_0)}{p_{\theta}(y_{0:t},x_{1:t}|x_0)}\right]:=\mathcal L_{VLB}.
\end{align}

First, we claim that the optimal $p_{\theta}(y_t|x_t)$ follows a Gaussian distribution up to a non-negative constant of \cref{eq:3.3}.

\begin{theorem}[Closed-form expression]\label{theorem:1}
The loss function in \cref{eq:3.3} has a closed-form representation.
The training is equivalent to optimizing a KL-divergence up to a non-negative constant, \textit{i.e.},
\begin{align}
\mathcal L_{VLB}=\mathbb E_{q(y_0,x_t|x_0)}\left[D_{KL}(q(y_t|y_0)\|p_{\theta}(y_t|x_t))\right] + C\sqq{.}\label{eq:3.4}
\end{align}
\end{theorem}

For the given closed-form expression in \cref{eq:3.4}, the optimal $p_{\theta}(y_t|x_t)$ follows a Gaussian distribution and its mean $\mu_{\theta}$ has an analytic form, as summarized in the \Cref{theorem:2} below:

\begin{theorem}[Optimal solution to \cref{eq:3.4}]\label{theorem:2}
The optimal $p_{\theta}(y_t|x_t)$ follows a Gaussian distribution with its mean being
\begin{align}
\mu_{\theta}(x_t) = \sqrt{\balpha{t}}y_0.
\end{align}
\end{theorem}

Detailed proofs of the above lemma and theorems are provided in Appendix B.

\subsection{Reparameterization of $\mu_{\theta}$}\label{subsec:Analytic}

Given the DDPM trained on the target domain, we first apply the same diffusion forward process on both $x_0$ and $y_0$ as a shared encoder \sqq{to represent the mean $\mu_{\theta}(x_t)$}:
\begin{align}\label{eq:3.6}
x_t=\sqrt{\balpha{t}}x_0+\sqrt{1-\balpha{t}}z_t,\quad y_t=\sqrt{\balpha{t}}y_0+\sqrt{1-\balpha{t}}z_t.
\end{align}

\Cref{theorem:2} reveals that $\mu_{\theta}$ needs to approximate the expression $\sqrt{\balpha{t}}y_0$ with $x_t$ as the only available input. Then\sqq{,} we apply the following parameterization,
\begin{align}\label{eq:3.5}
\mu_{\theta}(x_t)=f_{\theta}(x_t)-\sqrt{1-\balpha{t}}z(x_t),
\end{align}
where $f_{\theta}$ is a trainable function and $z(x_t)=z_t$\sqq{,} which is set to the shared noise component of $x_0$ and $y_0$. 
The KL-divergence in \cref{eq:3.4} is optimized by minimizing the difference between the two means together with the variance $\Sigma_{\theta}$ of $p_{\theta}(y_t|x_t)$.
Noting that $\Sigma_{\theta}=(1-\balpha{t})I$, the objective function then has the following form,
\begin{align}
\mathcal L_t
&=\mathbb E_q\left[\frac{1}{2(1-\balpha{t})}\|f_{\theta}(x_t)-y_t\|^2\right].
\end{align}

\cref{eq:3.5} implies that inferring $y_t\sim p_{\theta}(y_t|x_t)$ is to compute $f_{\theta}(x_t)-\sqrt{1-\balpha{t}}z_t+\sqrt{1-\balpha{t}}z$, where $z\sim\mathcal N(0,I)$.

\subsection{Determining an appropriate timestep for translation}\label{subsec:Optimal}

\definecolor{myblue}{RGB}{0, 0, 0}
\definecolor{myred}{RGB}{0, 0, 0}
\begin{figure*}[tp]
\centering
\begin{tabular}{cccc}
\begin{minipage}[t]{0.22\linewidth}\includegraphics[width=1\linewidth]{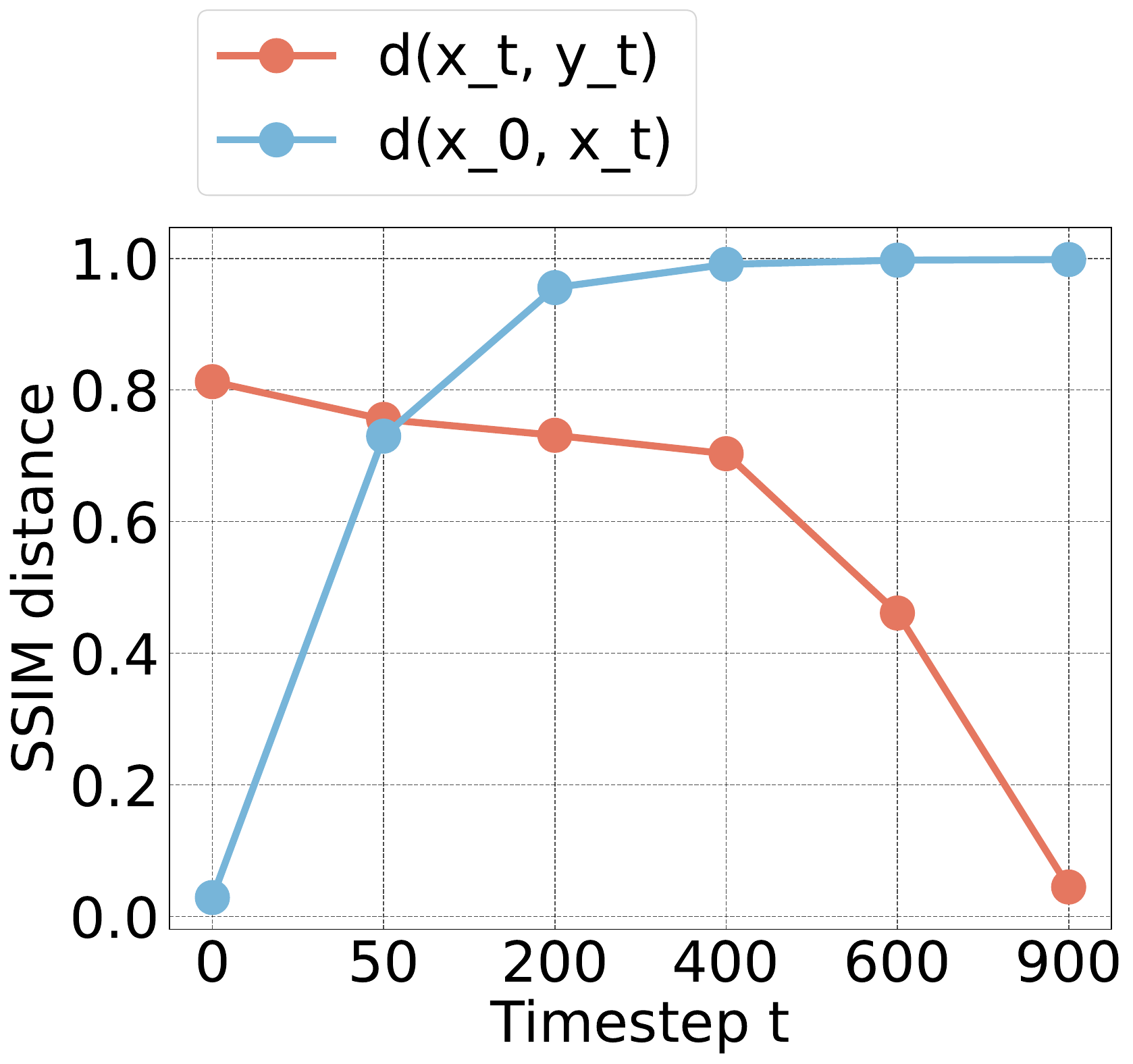}\end{minipage} & \begin{minipage}[t]{0.22\linewidth}\includegraphics[width=1\linewidth]{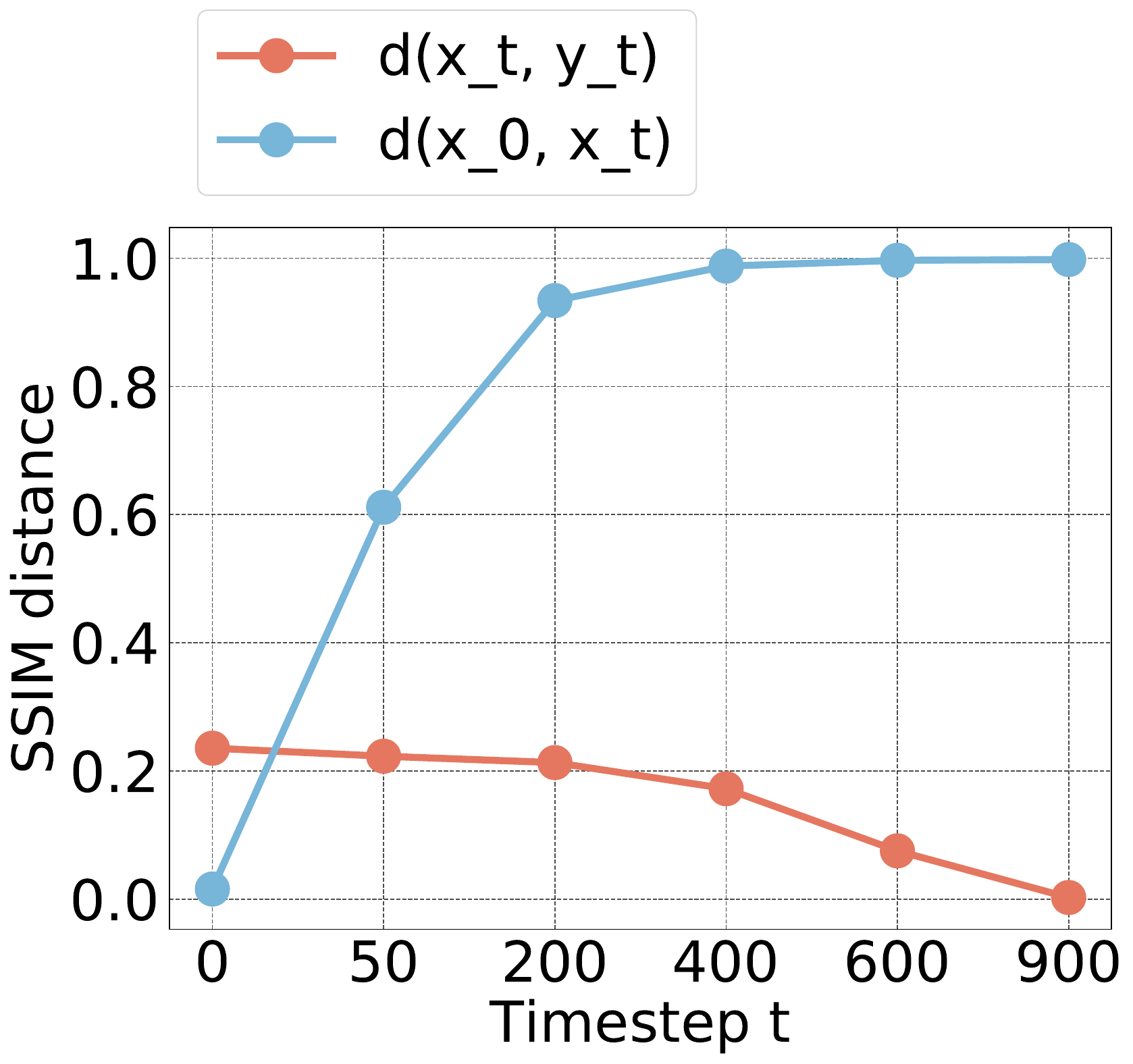}\end{minipage} &
\begin{minipage}[t]{0.22\linewidth}\includegraphics[width=1\linewidth]{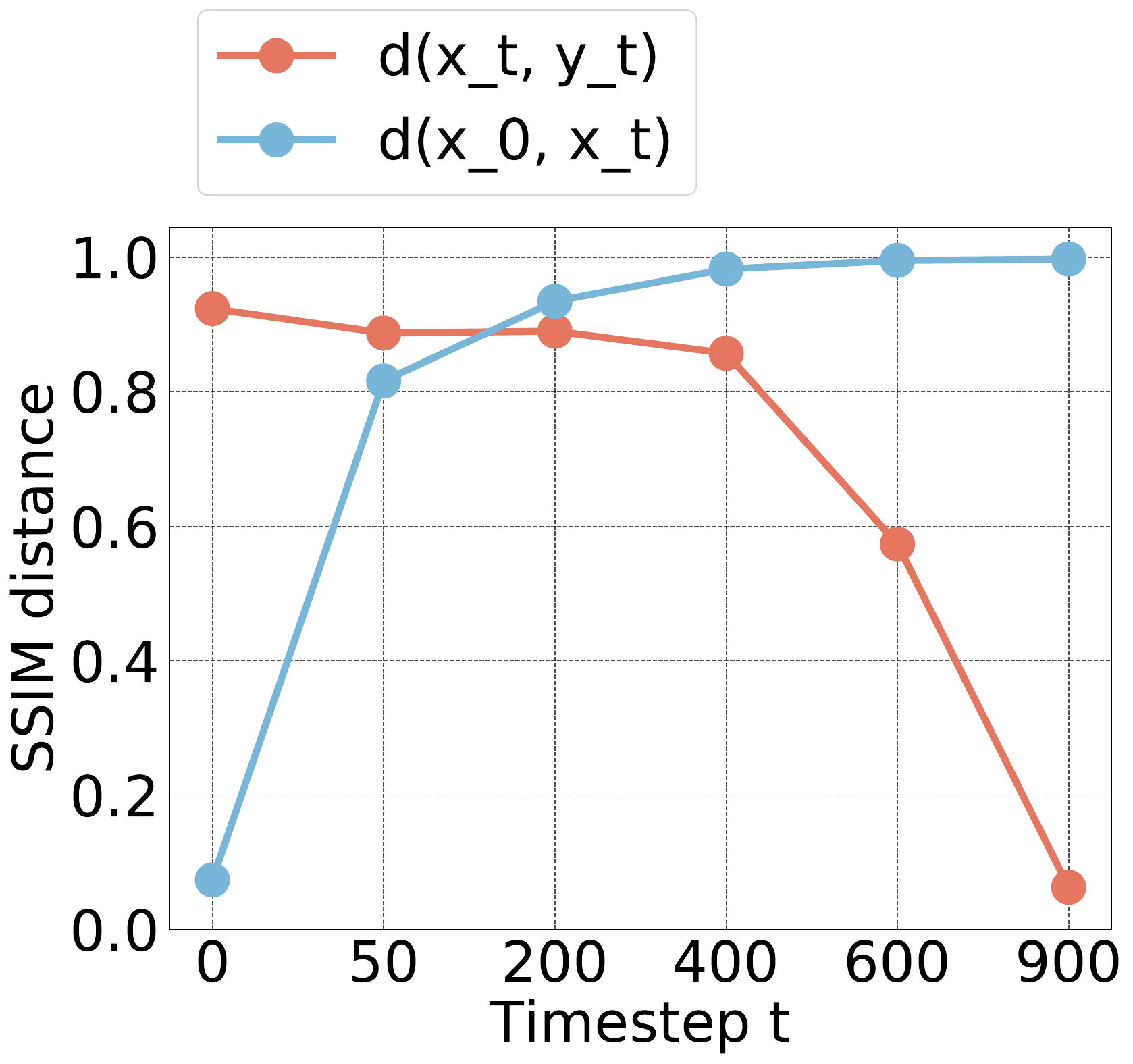}\end{minipage} & \begin{minipage}[t]{0.22\linewidth}\includegraphics[width=1\linewidth]{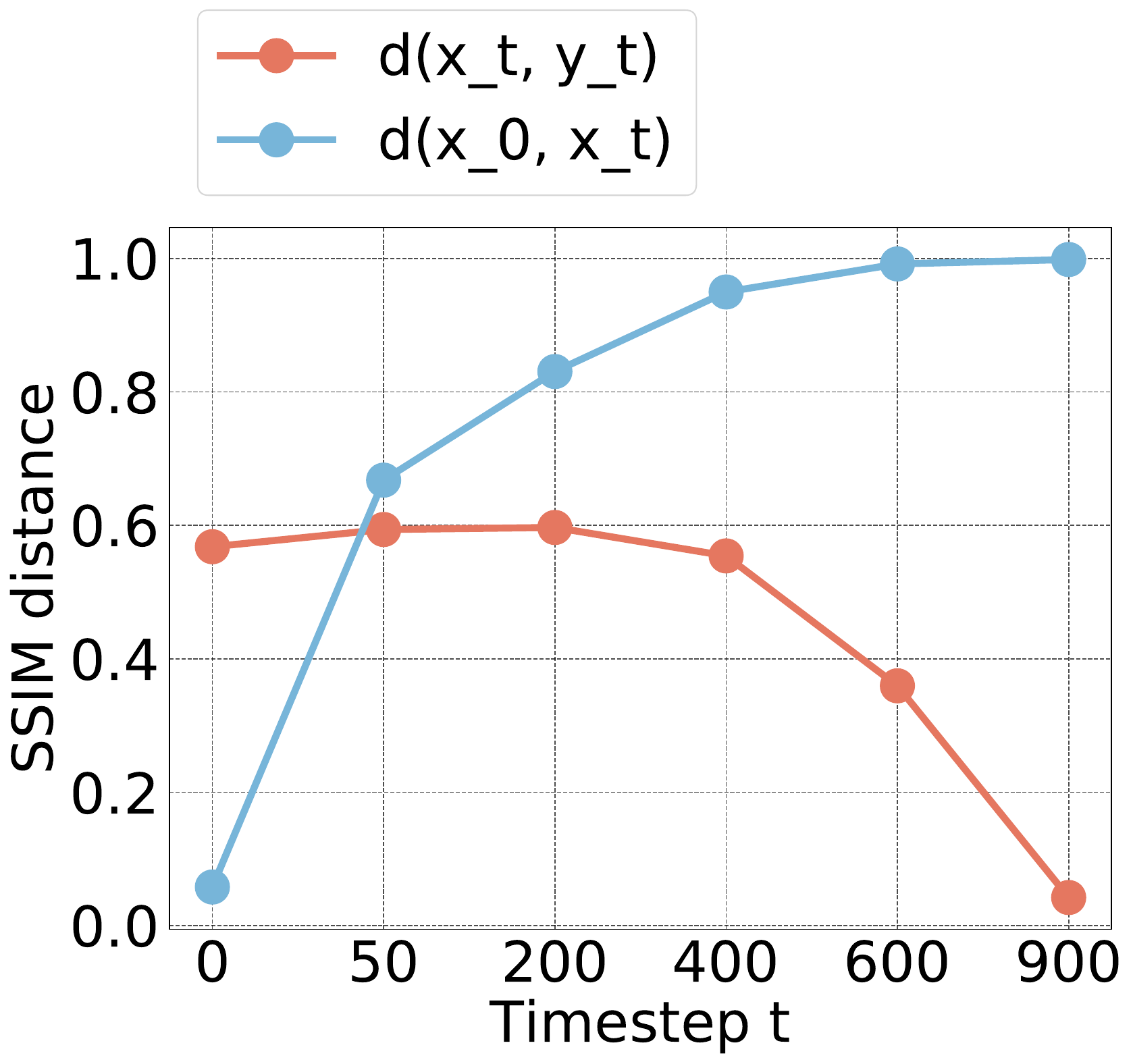}\end{minipage} \\
(a) Portrait & (b) AFHQ & (c) CelebA-HQ & (d) Edges2handbags
\end{tabular}
\caption{
    \textbf{Analysis on the preset timestep, $t$.}
    Our \method needs \revise{a} pre-defined timestep to learn and perform the distribution shift.
    We plot the distance between $(x_t,y_t)$ and $(x_0, x_t)$ at different timesteps, which are shown in red and blue curves, respectively.
    When $t$ increases, $d(x_t, y_t)$ decreases so that the distribution is easier to shift from $x_t$ to $y_t$, while $d(x_0, x_t)$ increases so that the \revise{input condition} signal is becoming less relevant because $x_t$ is drifting away from the input $x_0$.
    Considering such a trade-off, we select the intersection as the practical choice of the timestep for \method learning.
}
\label{fig:ablation_t_quantitative}
\vspace{-10pt}
\end{figure*}

Recall that we encode the same forward diffusion process onto both $x_0$ and $y_0$ using a shared encoder (ref. to \cref{eq:3.6}),
where $z_t$ is independent of $x_0$ and $y_0$.
As $t$ tends to $T$, $x_t$ and $y_t$ will converge to the same Gaussian noise simultaneously, since  $x_t,y_t\rightarrow z_T\sim\mathcal N(0,I)$.
Hence, as $t$ increases, the distance between $(x_t,y_t)$ will decrease and the distance between $(x_0,x_t)$ will increase.
In other words, the training of \method faces a trade-off between the gap between the two potential domains and the strength of the \revise{condition} signal.
The larger timestep $t$ makes it easier for the \method to learn the translation mapping, while the strength of inference information will be weakened since the injected noise corrupts the origin signal.

\newrevise{
To address this trade-off issue, we provide a theoretical analysis below.
Recall that our proposed diffusion-model-based I2I system consists of three sub-systems: (1) the forward diffusion process from $x_0$ to $x_t$, (2) \method from $x_t$ to $y_t$, and (3) the denoising process via pre-trained diffusion model from $y_t$ to $y_0$.
Our analysis is based on the following observation: the complexity $C$ of the whole system $S$ is determined by the maximal one among the complexities of three sub-systems $(S_1,S_2,S_3)$, \textit{i.e.}, $C(S)=\max\{C(S_1),C(S_2),C(S_3)\}$.
Given a timestep $t$, let $C(S_1)=f(t)$, $C(S_2)=g(t)$, $C(S_3)=h(t)$, where $f(t)$, $g(t)$ and $h(t)$ are complexity curves of diffusing $x_0$ to $x_t$, translating $x_t$ to $y_t$, and denoising $y_t$ to $y_0$ w.r.t. the timestep $t$, respectively.
First, we assume\footnote{\newrevise{This assumption is reasonable because the diffusion and denoising processes are reciprocal at the same time step, although in different domains.}} $f(t)\approx h(t)$.
Then $C(S)=\max\{f(t),g(t)\}$.
Second, we assume\footnote{\newrevise{This assumption is reasonable because the larger the time step, the greater the complexity of forward diffusion and the lower the complexity of \method.}} that $f(t)$ and $g(t)$ are monotone curves.
Then we have the conclusion that {\it $C(S)$ takes the minimum value at the intersection point of two monotone curves $f(t)$ and $g(t)$}.
}

\newrevise{
Accordingly, we propose a simple and effective strategy to determine an appropriate timestep $t$ before training.}
%
We calculate the \revise{$L_1$, $L_2$, Peak Signal-to-Noise Ratio (PSNR), Learned Perceptual Image Patch Similarity (LPIPS)~\cite{zhang2018unreasonable}, Fr\'{e}chet Inception Distance (FID)~\cite{heusel2017gans}, and }Structure Similarity Index Measure (SSIM)~\cite{wang2004image} between $(x_t,y_t)$ and between $(x_0, x_t)$\revise{, among which SSIM achieves the timestep with the best performance.}
\sqq{The results shown in \cref{fig:ablation_t_quantitative} are consistent with our aforementioned findings}: the distance between $(x_t,y_t)$ drops rapidly, while the distance between $(x_0, x_t)$ grows monotonically as the timestep $t$ grows.
Note that the intersection point of the two curves \newrevise{offers a good approximation for the minimum of system complexity.} 
This observation provides us with a pre-selecting strategy that chooses the timestep $t$ of this intersection point as an appropriate timestep $t$ for domain transfer\sqq{.}
We demonstrate in \cref{subsec:Ablation} the performance of using the timestep $t$ thus chosen by this pre-selecting method.

To summarize, we train the \method module in the same way as a simple I2I task.
First, we gradually apply the same diffusion forward process onto both the input \revise{condition} and the desired output until a pre-selected timestep.
Then\sqq{,} we train the function approximator $f_{\theta}$ using a reparameterization strategy to reformulate the objective function.
We theoretically prove the feasibility of the simple \method module and show that the approximator $f_{\theta}$ resembles the reverse process mean function approximator in DDPM~\cite{ho2020denoising}.
We \sqq{verify} the efficiency of the \method in \cref{sec:exp} with comprehensive experiments on a wide range of datasets, and provide the algorithms and the pseudo-codes in Appendix A.

\begin{figure*}[t]
\centering
\includegraphics[width=1.\textwidth]{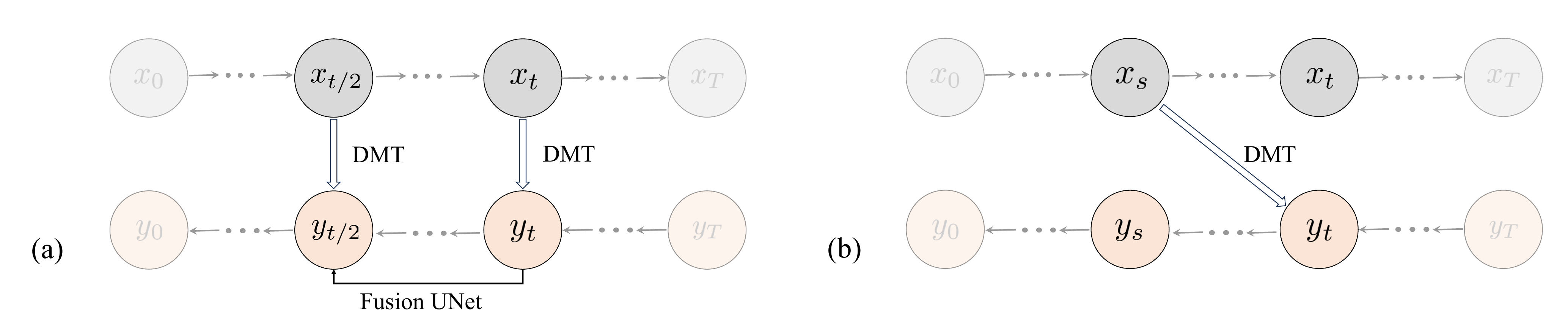}
\vspace{-15pt}
\caption{
    \newrevise{
    \textbf{Conceptual comparison} for (a) multi-step \method and (b) asymmetric \method.
    $\{x_t\}_{t=0}^T$ represent different states of the \revise{input from the source domain}, while $y_T \to y_0$ stands for the denoising process of DDPM.
    Here, $T$ denotes the total number of noise-adding steps in the diffusion process.
    Multi-step \method combines the translation results of \method at two different timesteps with an auxiliary fusion UNet and denoise to achieve the final output, while asymmetric \method applies translation at different timestep pair $(s,t)$.
    More discussions are addressed in \cref{subsec:generalization} and \supp.
    }
}
\label{fig:method_generalization}
\vspace{-10pt}
\end{figure*}

\subsection{\revise{Further discussion of \method}}\label{subsec:generalization}

\setlength{\tabcolsep}{10pt}
\begin{table*}[t]
\centering
\caption{
    \revise{
    \textbf{Ablation study} on the preset timestep pair $(s,t)$ in our proposed \method on the two I2I tasks under different $\lambda$ defined in \cref{eq:s_t}. 
    SSIM is used to evaluate the distance between samples.
    }
}
\label{tab:ablation_s_t}
\vspace{-2pt}
\begin{tabular}{cccc}
\toprule
\revise{Stylization} & \revise{Colorization} & \revise{Segmentation} & \revise{Sketch} \\
\midrule
\revise{$(s,t)=(50, 50)$} & \revise{$(s,t)=(5,5)$} & \revise{$(s,t)=(200, 200)$} & \revise{$(s,t)=(20, 20)$} \\
\bottomrule
\end{tabular}
\end{table*}

\newrevise{
Recall that} \revise{we introduce the shared encoder by diffusing both $x_0$ and $y_0$ with the identical timestep $t$.
\sqq{To} address the trade-off between the strength of content information and domain gap, we propose a strategy to automatically preset an adequate timestep $t^*$ to achieve equilibrium between the distances of $(x_0,x_t)$ and $(x_t,y_t)$.
}
\newrevise{
Therefore, one could reasonably consider to use \yr{(1)} multi-step translation results from \method to facilitate the denoising precess, or \yr{(2)} diffusion processes with distinct timesteps for the source and target domains\yr{,} as a strategy to mitigate trade-offs and achieve improved performance.
In this subsection, we discuss these two interesting alternatives, by fusing the \method results at \yr{multiple} timestep\yr{s} \yr{(\textit{e.g.},} $t$ and $t/2$) (\textit{i.e.}, \cref{fig:method_generalization} (a)), together with \sqq{using} the \textit{asymmetric} timestep pair $(s,t)$ (\textit{i.e.}, \cref{fig:method_generalization} (b)), \sqq{where} $x_0$ and $y_0$ \sqq{are diffused at timesteps} $s$ and $t$, $s\neq t$ respectively.
Given the results analyzed in this section, we conclude that the former multi-step method significantly increases training time cost while degrading the FID performance, and that the latter more complicated pipeline practically coincides with our proposed \method method, since the optimal timestep pair $(s,t)$ appears to be the same.
}

\newrevise{
To implement the multi-step \method, due to the use of the vanilla DDPM, which is only capable of inputting a 3-channel input intermediate noisy image, we train an auxiliary UNet model to fuse the $y_{t/2}$ transformed from $x_{t/2}$ together with the $y'_{t/2}$ denoised from the $y_{t}$.
However, we argue that the additional UNet significantly increases the training cost, while degrading the FID performance, due to additional error from the UNet.
Detailed experimental setups and quantitative comparison are provided in \supp.
}

\newrevise{
As for the asymmetric setting,
}
\revise{
\sqq{we define} the disjoint distribution of the forward Markov process of $(x_{1:s}, y_{1:t})$ as below:
\begin{align}
q(y_{1:t},x_{1:s}|y_0, x_0)& =\prod_{i=1}^s q(x_i|x_{i-1})\prod_{j=1}^t q(y_j|y_{j-1}), \label{eq:3.7} \\
p_{\theta}(y_{0:t},x_{1:s}|x_0)&=p_{\theta}(y_t|x_s)\prod_{i=1}^s q(x_i|x_{i-1})\prod_{j=1}^t q(y_{j-1}|y_j)\sqq{.} \label{eq:3.x}
\end{align}
We first claim the feasibility of this pipeline, whose proofs are addressed in \supp.
Similar to \Cref{lem:1}, \Cref{theorem:1,theorem:2}, we have
\begin{lemma}\label{lem:2}
The negative log-likelihood of $-\log p_{\theta}(y_0|x_0)$ has the following upper bound,
\begin{align}
-\log p_{\theta}(y_0|x_0)\leqslant\mathbb E_q\left[\log\frac{q(y_{1:t},x_{1:s}|y_0, x_0)}{p_{\theta}(y_{0:t},x_{1:s}|x_0)}\right],
\end{align}
where $q=q(y_{1:t},x_{1:s}|y_0, x_0).$
\end{lemma}
We accordingly define the $\mathcal L_{VLB}$ as below:
\begin{align}\label{eq:3.8}
\mathcal L_{CE}&=\mathbb -\mathbb E_{q(y_0|x_0)}\left[\log p_{\theta}(y_0|x_0)\right]\\
&\leqslant\mathbb E_{q(y_{0:t},x_{1:s}|x_0)}\left[\log\frac{q(y_{1:t},x_{1:s}|y_0, x_0)}{p_{\theta}(y_{0:t},x_{1:s}|x_0)}\right]:=\mathcal L_{VLB}.
\end{align}
Then we have the re-claimed \Cref{theorem:1}:
\begin{theorem}[Closed-form expression]\label{theorem:3}
The loss function in \cref{eq:3.8} has a closed-form representation.
The training is equivalent to optimizing a KL-divergence up to a non-negative constant, \textit{i.e.},
\begin{align}
\mathcal L_{VLB}=\mathbb E_{q(y_0,x_s|x_0)}\left[D_{KL}(q(y_t|y_0)\|p_{\theta}(y_t|x_s))\right] + C\sqq{.}\label{eq:3.9}
\end{align}
\end{theorem}
For the given closed-form expression in \cref{eq:3.9}, the optimal $p_{\theta}(y_t|x_s)$ follows a Gaussian distribution and its mean $\mu_{\theta}$ has an analytic form, as summarized in the \Cref{theorem:2} \sqq{above.}
\begin{theorem}[Optimal solution to \cref{eq:3.9}]\label{theorem:4}
The optimal $p_{\theta}(y_t|x_s)$ follows a Gaussian distribution with its mean being
\begin{align}
\mu_{\theta}(x_s) = \sqrt{\balpha{t}}y_0\sqq{.}
\end{align}
\end{theorem}
By applying the diffusion forward process on both $x_0$ and $y_0$ with identical random noise at asymmetric timestep $s$ and $t$, respectively, \sqq{we have the following:}
\begin{align}
x_s=\sqrt{\balpha{s}}x_0+\sqrt{1-\balpha{s}}z,\quad y_t=\sqrt{\balpha{t}}y_0+\sqrt{1-\balpha{t}}z.
\end{align}
\Cref{theorem:4} reveals that $\mu_{\theta}$ needs to approximate the expression $\sqrt{\balpha{t}}y_0$ with $x_s$ as the only available input. Then we apply the following parameterization,
\begin{align}\label{eq:3.10}
\mu_{\theta}(x_s)=f_{\theta}(x_s)-\sqrt{1-\balpha{t}}z,
\end{align}
where $f_{\theta}$ is a trainable function.
The KL-divergence in \cref{eq:3.9} is optimized by minimizing the difference between the two means together with the variance $\Sigma_{\theta}$ of $p_{\theta}(y_t|x_s)$.
Formally, we have the simplified objective:
\begin{align}
\mathcal L_{s,t}=\mathbb E_q\left[\frac{1}{2(1-\balpha{t})}\|f_{\theta}(x_s)-y_t\|^2\right].
\end{align}
}

\revise{
To determine an adequate timestep pair $(s,t)$} \newrevise{for the asymmetric diffusion process, similar to the theoretical analysis about original \method, the complexity of our I2I system is characterized by $C(S)=\max\{C(S_1),C(S_2),C(S_3)\}$.
For I2I with the asymmetric \method, the three sub-systems are (1) the forward diffusion process from $x_0$ to $x_s$ with the complexity $f(s)$, (2) \method from $x_s$ to $y_t$ with the complexity $g(s,t)$, and (3) the denoising process via pre-trained diffusion model from $y_t$ to $y_0$ with the complexity $h(t)$.
$f(s)$ and $h(t)$ are monotone w.r.t. $s$ and $t$, respectively; but $g(s,t)$ does not have to be monotone.
If $s\neq t$, the diffusion process from $x_0$ to $x_s$ and denoising process from $y_t$ to $y_0$ are no longer reciprocal, so we need to consider both $f(s)$ and $h(t)$.
Then the complexity of $C(S)$ can be represented as $C(S)=C(s,t)=\max\{f(s),g(s,t),h(t)\}$.
Our target is to search the timestep pair $(s,t)$ minimizing $\min_{s,t}C(s,t)$.
We have
\begin{align}
\max_{i=1,2,3}d_i&= \max\{\max\{d_1,d_2\},d_3\}\\
&=\max\{\frac{d_1+d_2}{2}+\frac{|d_1-d_2|}{2},d_3\} \\
&\geqslant\max\{\frac{d_1+d_2}{2},d_3\} \\
&\geqslant\frac{1}{3}(2\cdot\frac{d_1+d_2}{2}+d_3)=\frac{1}{3}(d_1+d_2+d_3),
\end{align}
where the equality holds if and only if $|d_1-d_2|=0$ and $\frac{d_1+d_2}{2}=d_3$, \textit{i.e.}, $d_1=d_2=d_3$.
That means $C(s,t)=\max\{f(s),g(s,t),h(t)\}$ reaches its minimum when $s=t$.
In practice, we add the regularity term $\mathrm{SSIM}(x_0,x_s)+\mathrm{SSIM}(x_s,y_t)+\mathrm{SSIM}(y_0,y_t)$ to help search the global minimum.}
\newrevise{
Formally, we calculate the weighted sum of SSIM distances defined below, in which the smaller the result the better the performance.
}
\revise{
\begin{align}\label{eq:s_t}
\mathrm{dist}(s,t)=&|\mathrm{SSIM}(x_0,x_s)-\mathrm{SSIM}(x_s,y_t)|\nonumber\\
&\qquad+|\mathrm{SSIM}(x_s,y_t)-\mathrm{SSIM}(y_0,y_t)|\nonumber\\
&\qquad+|\mathrm{SSIM}(x_0,x_s)-\mathrm{SSIM}(y_0,y_t)|\nonumber\\
&\qquad+\lambda\mathrm{SSIM}(x_0,x_s)\nonumber\\
&\qquad+\lambda\mathrm{SSIM}(x_s,y_t)\nonumber\\
&\qquad+\lambda\mathrm{SSIM}(y_0,y_t)).
\end{align}
By setting the weight $\lambda=0.5$, we acquire an appropriate timestep pair as in \cref{tab:ablation_s_t}.
Notably, the preset timestep pair $(s,t)$ of this generalized pipeline coincide with the original pipeline theoretically and empirically, \textit{i.e.}, the asymmetric timestep pair appears to be identical.
}

\section{Experiments}\label{sec:exp}

In this section, we evaluate the proposed \method \sqq{on} four different I2I tasks: image stylization, colorization, segmentation to image, and sketch to image.
We first show that the \method is capable of mapping translation between the two domains of an I2I task in \cref{subsec:Quality}.
Then\sqq{,} we compare the DMT with \sqq{several} representative methods to demonstrate its superior efficiency and performance in \cref{subsec:Comparison}.
Finally, we provide an ablation study on the effect of the timestep $t$ for training in \cref{subsec:Ablation}.

\subsection{Experimental setups}\label{subsec:Setup}

\noindent\textbf{Datasets and tasks.}
We train the I2I task on four datasets: our handcrafted Portrait dataset using CelebA-HQ by QMUPD~\cite{YiLLR22}, AFHQ~\cite{choi2020stargan}, CelebA-HQ~\cite{karras2018progressive}, and Edges2handbags~\cite{zhu2016generative,xie15hed}.
All the images are resized to $256\times256$ resolution.
Our Portrait dataset consists of 27,000 \sqq{images for training} and 3,000 images for inference; all these images are generated from \sqq{the} CelebA-HQ dataset using \sqq{a} pretrained QMUPD model.
The AFHQ dataset consists of 14,630 images for training and 1,500 images for inference, \sqq{encompassing a variety of cats, dogs, and wild animal images.}
For \sqq{the} CelebA-HQ dataset, we randomly choose 27,000 images together with their segmentation masks as the paired training data, while the remaining 3,000 images are used as test data.
As for Edges2handbags, we use \sqq{all} 138,567 images as training data and the 200-image test data for inference.

\noindent\textbf{Evaluation metrics.} 
We use Fr\'{e}chet Inception Distance (FID)~\cite{heusel2017gans}, Structure Similarity Index Measure (SSIM)~\cite{wang2004image}, LPIPS~\cite{zhang2018unreasonable}, $L1$ and $L2$ metrics to evaluate the fidelity of the generated images and how well the content information is kept after the translation.
Besides, we compare all the methods in a user study, where users were asked to score the image quality from 1 to 5.
We also compare the training and inference efficiency of all the methods by comparing the number of total training epochs, training speed for 1,000 images, and inference time for generating an image.

\noindent\textbf{Baselines.}
We compare our proposed \method algorithm with \revise{five} representative I2I algorithms:  Pix2Pix~\cite{isola2017image}, TSIT~\cite{jiang2020tsit}, SPADE~\cite{park2019SPADE}, QMUPD~\cite{YiLLR22}, and Palette~\cite{saharia2021palette}.
\revise{The alternatives can be divided into two categories: GAN-based and DDPM-based algorithms.}
Pix2Pix is a classic cGAN-based method involving $L_1$ and adversarial loss.
TSIT is a \revise{GAN-based} versatile framework using specially designed normalization layers and coarse-to-fine feature transformation.
SPADE is a \revise{GAN-based} specially-designed framework for semantic image synthesis with spatially-adaptive normalization.
QMUPD \revise{is also GAN-based, which} is specially designed for portrait stylization by unpaired training. \sqq{We train} the model with paired data for fair comparison.
Palette introduces the DDPM~\cite{ho2020denoising} framework into the I2I task and injects the input constraint to each step of \sqq{the} denoising process.

\noindent\textbf{Implementation details.}
We train the proposed \method module on the platform of PyTorch~\cite{paszke2019pytorch}, in a Linux environment with an NVIDIA Tesla A100 GPU.
We set total timestep $T=1000$ for all the experiments, the same setting as in~\cite{ho2020denoising}.
We train the reverse denoising process of the DDPM using a U-Net backbone together with the Transformer sinusoidal embedding~\cite{ronneberger2015u,vaswani2017attention}, following~\cite{dhariwal2021diffusion}.
The DDPM is frozen during the training of the \method module.
To train the \method module, we use the Pix2Pix~\cite{isola2017image} and TSIT~\cite{jiang2020tsit} model.  We remove the discriminator model and train only the generator block to ensure that the translator $f_{\theta}$ has approximately the same functional form as the real mapping.
Note that our \method employs the DDPM denoising process during sampling, which employs hundreds of iterative function evaluations for denoising and can be \sqq{time-consuming}.
Therefore, we apply DDIM~\cite{song2020denoising} for acceleration, which realizes high-quality synthesis within \sqq{10 function evaluations} (NFE = 10).

\setlength{\tabcolsep}{5pt}
\begin{table}[!ht]
\small
\renewcommand\arraystretch{1}
\centering
\vspace{-10pt}
\caption{
    \textbf{Quantitative comparison} between \method and SPADE~\cite{park2019SPADE} on segmentation-to-image task.
    FID, SSIM, LPIPS, L1, and L2 metrics are used to evaluate the image quality and content consistency, respectively.
}
\label{tab:spade}
\vspace{-8pt}
\begin{tabular}{l|cccccc}
\toprule
Method & \textbf{FID$\downarrow$} & \textbf{SSIM$\uparrow$} & \textbf{LPIPS$\downarrow$} & \textbf{L1$\downarrow$} & \textbf{L2$\downarrow$} \\
\midrule
SPADE \revise{(GAN)}    & 66.55 & 0.140 & 0.487  & 0.413  & 0.285 \\
Ours & \bf 36.78 & \bf 0.446 & \bf 0.433  & \bf 0.182  & \bf 0.053 \\
\bottomrule
\end{tabular}
\end{table}
\begin{figure}[!ht]
\vspace{-8pt}
\centering
\includegraphics[width=0.5\textwidth]{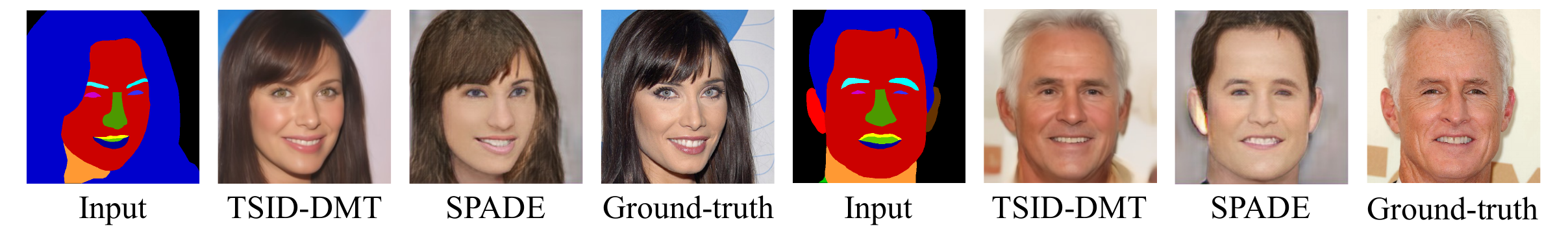}
\vspace{-10pt}
\caption{
    \textbf{Qualitative comparison} between \method and SPADE~\cite{park2019SPADE} on segmentation-to-image task.
    Our proposed \method achieves better image quality and content consisitency compared with SPADE.
}
\label{fig:spade}
\end{figure}

\subsection{Qualitative evaluation on various tasks}\label{subsec:Quality}

\setlength{\tabcolsep}{5pt}
\begin{table}[!ht]
\small
\renewcommand\arraystretch{1}
\centering
\vspace{-10pt}
\caption{
    \textbf{Quantitative comparison} between \method and QMUPD~\cite{YiLLR22} on image stylization task.
    FID, SSIM, LPIPS, L1, and L2 metrics are used to evaluate the image quality and content consistency, respectively.
}
\label{tab:qmupd}
\vspace{-8pt}
\begin{tabular}{l|cccccc}
\toprule
Method & \textbf{FID$\downarrow$} & \textbf{SSIM$\uparrow$} & \textbf{LPIPS$\downarrow$} & \textbf{L1$\downarrow$} & \textbf{L2$\downarrow$} \\
\midrule
QMUPD \revise{(GAN)}    &     12.81 &     0.660 &     0.248 &     0.268 &     0.392 \\
Ours     & \bf 11.01 & \bf 0.760 & \bf 0.138 & \bf 0.131 & \bf 0.101 \\
\bottomrule
\end{tabular}
\end{table}
\begin{figure}[!ht]
\vspace{-8pt}
\centering
\includegraphics[width=0.5\textwidth]{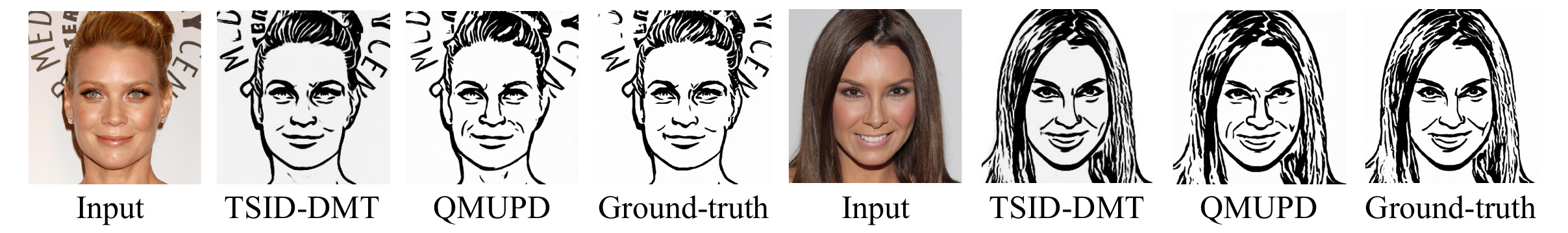}
\vspace{-10pt}
\caption{
    \textbf{Qualitative comparison} between \method and QMUPD~\cite{YiLLR22} on image stylization task.
    Our proposed \method achieves better image quality and content consisitency compared with QMUPD.
}
\label{fig:qmupd}
\end{figure}

\begin{figure*}[t]
\centering
\includegraphics[width=1.0\textwidth]{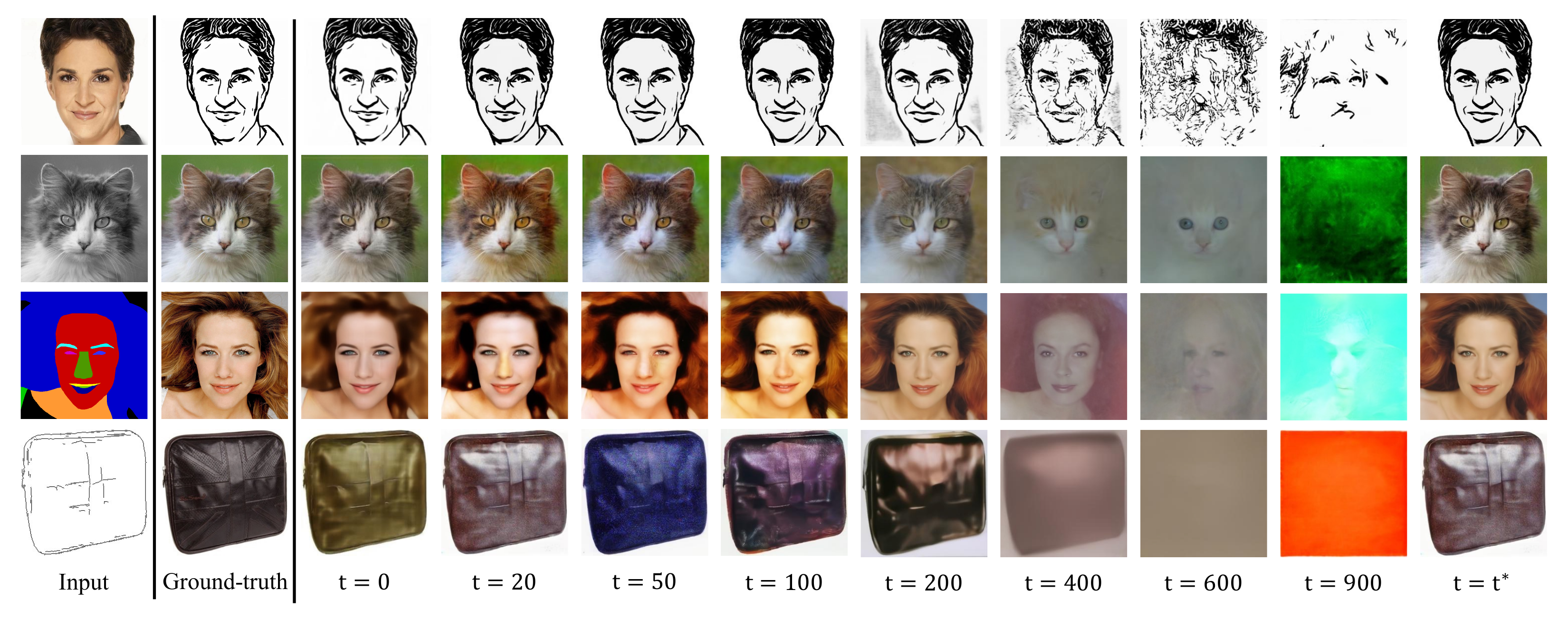}
\vspace{-10pt}
\caption{
    \textbf{Qualitative results for ablation study} of the preset timestep $t$ in our proposed \method on the four I2I tasks.
    \sqq{We observe that a smaller $t$ helps in better retaining the content information from the input source,} but suffers from a larger gap between the target domain and the source domain. 
    The optimally selected timestep ($t^*$) for each of the four I2I tasks is given in \cref{tab:quantitative_ablation}.
}
\label{fig:ablation_t_qualitative}
\vspace{-5pt}
\end{figure*}

\setlength{\tabcolsep}{1pt}
\begin{table*}[!ht]
\footnotesize
\centering
\caption{
    \textbf{Quantitative comparison} between Palette~\cite{saharia2021palette}, Pix2Pix~\cite{isola2017image}, TSIT~\cite{jiang2020tsit}\sqq{,} and our proposed \method.
    FID, SSIM\sqq{,} and LPIPS are used to evaluate the image quality and content preservation, respectively.
    Besides, we introduce the user study ({\bf Score}) to evaluate the quality of the synthesized images.
    We also report the total number of training epochs ({\bf Ep.}), training time for 1,000 images ({\bf Train}), and inference time for a single image ({\bf Infer.}) of each method.
    %
}
\label{tab:quantitative_comparison}
\vspace{-2pt}
\begin{tabular}{l|cccc|cccc|cccc|cccc|ccc}
\toprule
& \multicolumn{4}{c|}{Stylization}& \multicolumn{4}{c|}{Colorization} & \multicolumn{4}{c|}{Segmentation} & \multicolumn{4}{c|}{Sketches} & & & \\
\midrule
Method & \textbf{FID$\downarrow$} & \textbf{SSIM$\uparrow$} & \textbf{LPIPS$\downarrow$} & \textbf{Ep.} & \textbf{FID$\downarrow$} & \textbf{SSIM$\uparrow$} & \textbf{LPIPS$\downarrow$} & \textbf{Ep.} & \textbf{FID$\downarrow$} & \textbf{SSIM$\uparrow$} & \textbf{LPIPS$\downarrow$} & \textbf{Ep.} & \textbf{FID$\downarrow$} & \textbf{SSIM$\uparrow$} & \textbf{LPIPS$\downarrow$} & \textbf{Ep.} & \textbf{Train} & \textbf{Infer.} & \textbf{Score} \\
\midrule
Palette \revise{(DDPM)}                &     17.16 &     0.663 &     0.366 & 2500 &     14.48 &     0.582 &     0.299 & 2450 &     40.77 &     0.092 &     0.521 & 1000 &     74.51 &     0.360 &     0.275 & 215  &  71s & 21.63s &     3.0 \\
\midrule
Pix2Pix \revise{(GAN)}               &     19.14 &     0.630 &     0.260 &   60 &     17.50 & \bf 0.769 &     0.263 &   60 &     70.98 &     0.105 &     0.542 &   60 &     77.80 &     0.524 &     0.306 &   40 &  25s &  0.09s &     1.7 \\
Pix2Pix-\method (Ours) & \bf 10.81 & \bf 0.703 & \bf 0.183 &   60 & \bf 17.44 &     0.752 & \bf 0.263 &   60 & \bf 65.26 & \bf 0.137 & \bf 0.534 &   60 & \bf 76.75 & \bf 0.527 & \bf 0.306 &   40 &  20s &  0.31s & \bf 3.5 \\
\midrule
TSIT \revise{(GAN)}                  &     16.62 &     0.681 &     0.235 &   60 &     13.60 &     0.645 &     0.243 &   60 &     40.59 &     0.357 &     0.450 &   60 &     76.80 &     0.606 &     0.282 &   40 & 134s &  0.11s &     3.6 \\
TSIT-\method (Ours)    & \bf 11.01 & \bf 0.760 & \bf 0.138 &   60 & \bf 13.03 & \bf 0.684 & \bf 0.180 &   60 & \bf 36.78 & \bf 0.446 & \bf 0.433 &   60 & \bf 74.37 & \bf 0.687 & \bf 0.255 &   40 &  82s &  0.48s & \bf 4.4 \\
\bottomrule
\end{tabular}
\end{table*}

The process of inferring images \sqq{with} \method consists of the following three simple steps.
\begin{itemize}
\item[1)] We apply the forward diffusion process to the input image $x_0$ until the pre-selected timestep $t$ to obtain $x_t$, which can be written as $x_t=\sqrt{\balpha{t}}x_0+\sqrt{1-\balpha{t}}z_t$;
\item[2)] By obtaining the mean by the functional approximator $f_{\theta}$ according to \cref{eq:3.5}, we infer the approximated $y_t$ by adding another Gaussian noise;
\item[3)] Using $y_t$ as the intermediate result, sampling with the given pre-trained DDPM by the reverse process achieves the required output.
\end{itemize}

We conducted four experiments to evaluate our proposed \method on four datasets, \textit{i.e.}, our handcrafted Portrait dataset, AFHQ~\cite{choi2020stargan}, CelebA-HQ~\cite{karras2018progressive}, and Edges2handbags~\cite{zhu2016generative,xie15hed}.
In training\sqq{,} we use 40 epochs for the sketch-to-image task, and 60 epochs for the other three tasks.
As shown in \cref{fig:teaser}, our method is capable of learning the cross-domain translation mapping and generates high-quality images.
For example, in the stylization task, the shared encoder is able to distinguish the two different forward diffusion processes of the two domains.
In the other tasks, our method can still extract the \revise{input} feature and generate photo-realistic images with high diversity \sqq{even with little input \revise{condition} information}
More results can be found in Appendix C.

\begin{figure*}[!ht]
\centering
\includegraphics[width=0.98\textwidth]{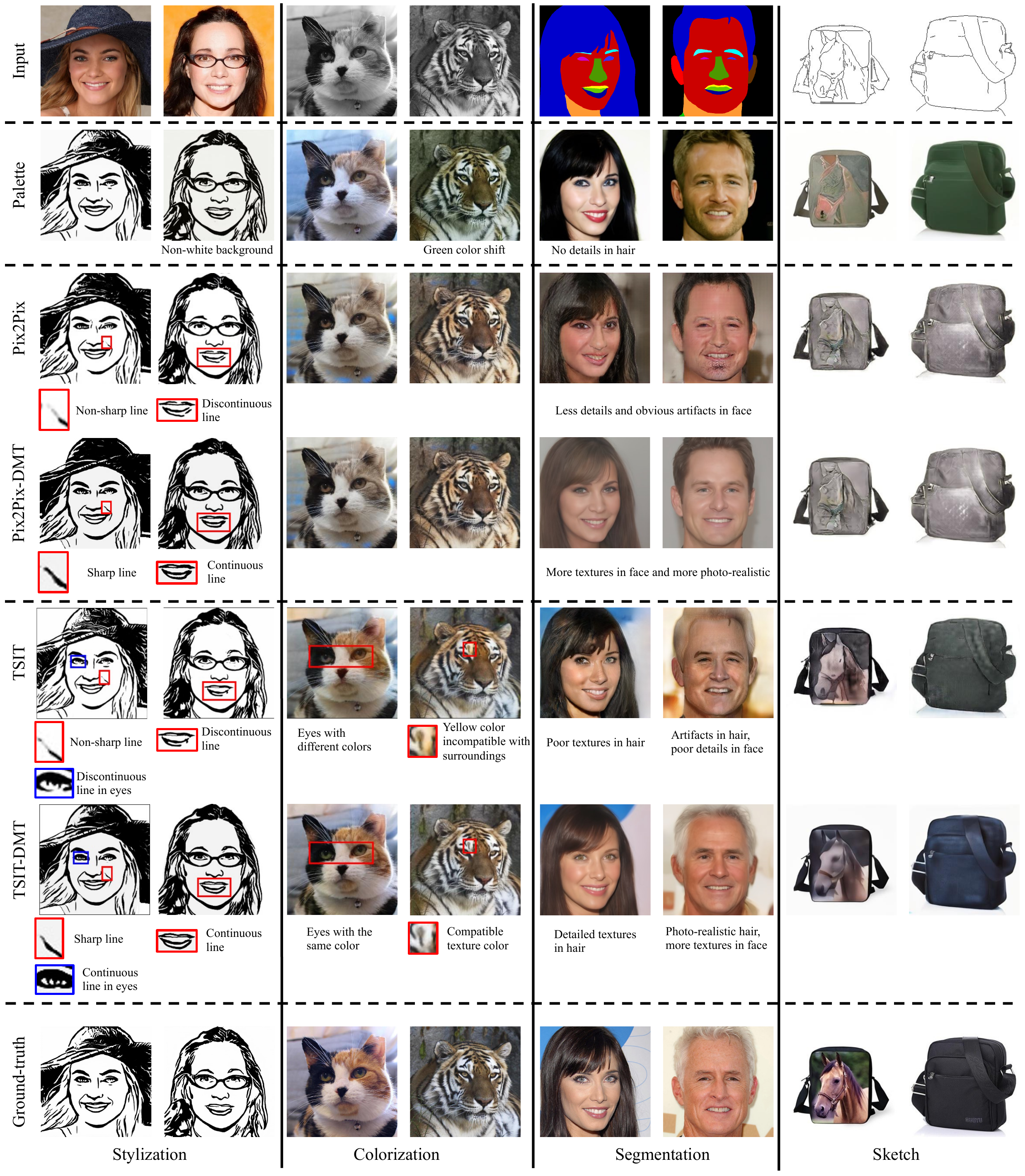}
\vspace{-15pt}
\caption{
    \newrevise{
    \textbf{Qualitative comparison}. Our \method achieves on par or better results than the three baseline methods Pix2Pix~\cite{isola2017image}, Palette~\cite{saharia2021palette}, TSIT~\cite{jiang2020tsit} on the four I2I tasks, \sqq{which are} image stylization, image colorization, segmentation to image, and sketch to image. Significant differences are highlighted in red or blue boxes, and brief textual explanations are provided \yr{besides the boxes}.
    The comparison on efficiency can be found in \cref{tab:quantitative_comparison}.}
    %
}
\label{fig:comparison}
\vspace{-5pt}
\end{figure*}

\subsection{Comparisons}\label{subsec:Comparison}

We qualitatively and quantitatively compare our method with the four classic I2I methods:  Pix2Pix~\cite{isola2017image}, TSIT~\cite{jiang2020tsit}, SPADE~\cite{park2019SPADE}, QMUPD~\cite{YiLLR22}, and the DDPM-based conditional generation method Palette~\cite{saharia2021palette}.
First, we compare with SPADE~\cite{park2019SPADE}.
It requires category-wise segmentation masks, limiting its application \sqq{to} most I2I tasks.
Note that our proposed \method introduces the shared encoder by gradually adding noise onto the original images, which \sqq{corrupts} the semantic information from the category-wise segmentation masks.
Hence\sqq{,} we only compare with SPADE on segmentation-to-image task, without applying the \method on top of it.

\sqq{T}he results are shown in \cref{fig:spade} and \cref{tab:spade}.

We also compare with the specially-designed stylization algorithm QMUPD~\cite{YiLLR22}.
It introduces a quality metric guidance for portrait generation using unpaired training data.
We train QMUPD with paired data for fair comparison, which reduces the training difficulty and achieves a stronger baseline. 
\sqq{The results, presented in \cref{fig:qmupd} and \cref{tab:qmupd}, demonstrate that our approach achieves performance that is on par with, or even surpasses, existing standards.}

Then, we compare with Palette~\cite{saharia2021palette} using the open source implementation\footnote{https://github.com/Janspiry/Palette-Image-to-Image-Diffusion-Models}.
As shown in \cref{fig:comparison}, we observe that the results of Palette fail to extract the segmentation feature of CelebA-HQ and Edges2handbags dataset\sqq{. Consequently, this leads to an inability to accurately generate details in the background of human images or replicate the horse pattern on the bags.}
%
As a comparison, our proposed \method can generate \sqq{high-quality} images \sqq{and preserve} the semantic information of the input \revise{condition}, even \sqq{when} given little input semantic information.

Next\sqq{,} we compare with Pix2Pix~\cite{isola2017image}.
We observe that our method can generate images of much higher quality than the Pix2Pix method.
For instance, the generated images of Pix2Pix suffer from severe artifacts over the facial region in the CelebA-HQ datasets, while our method consistently produces high\sqq{-}quality results.
Moreover, the \sqq{feature extraction performance} is significantly improved by the shared encoder and the well-prepared DDPM model in our method.

We finally compare with TSIT~\cite{jiang2020tsit}.
Although TSIT introduces a coarse-to-fine feature transformation block and hence can synthesize high-quality images in most cases, it fails to produce results with sufficient and satisfying semantics and textures when given very little inference information (\textit{e.g.}, hair and forehead region of segmentation).
In contrast, the results of \method have clear \sqq{boundaries} at the forehead and hair region, together with rich texture.

The quantitative results are reported in \cref{tab:quantitative_comparison}, showing that our method has the best image fidelity (FID), the lowest perceptual loss (LPIPS), and comparable structural similarity (SSIM).
Regarding the training and inference speed, our method uses the smallest number of training epochs and has \sqq{the} fastest training speed for generating 1,000 images, because it only needs to train one translation module.
The \method is also 40x $\sim$ 80x faster than Palette~\cite{saharia2021palette} due to starting the sampling process at an intermediate step (4x $\sim$ 8x faster) and the use of the fast sampling algorithm DDIM ($\sim$10x faster).

\subsection{Ablation study on the timestep for domain translation}\label{subsec:Ablation}

\setlength{\tabcolsep}{12pt}
\setlength{\tabcolsep}{10pt}
\begin{table*}[t]
\centering
\caption{
    \textbf{Ablation study} on the preset timestep $t$ in our proposed \method on the four I2I tasks.
    FID and SSIM are used to evaluate the image quality and content preservation, respectively.
    %
    %
    \sqq{Notably,} the model trained on the timestep $t^*$, which is automatically selected by our strategy in \cref{subsec:Optimal}, achieves satisfactory performance for all the four tasks.
}
\label{tab:quantitative_ablation}
\vspace{-2pt}
\begin{tabular}{l|cc|cc|cc|cc}
\toprule
& \multicolumn{2}{c|}{Stylization} & \multicolumn{2}{c|}{Colorization} & \multicolumn{2}{c|}{Segmentation} & \multicolumn{2}{c}{Sketches} \\
& \multicolumn{2}{c|}{($t^*=50$)} & \multicolumn{2}{c|}{($t^*=5$)} & \multicolumn{2}{c|}{($t^*=200$)} & \multicolumn{2}{c}{($t^*=20$)} \\
\midrule
Method & \textbf{FID$\downarrow$} & \textbf{SSIM$\uparrow$} & \textbf{FID$\downarrow$} & \textbf{SSIM$\uparrow$} & \textbf{FID$\downarrow$} & \textbf{SSIM$\uparrow$} & \textbf{FID$\downarrow$} & \textbf{SSIM$\uparrow$} \\
\midrule
TSIT-\method ($t=0$)   &         22.38  & \textbf{0.827} &         13.28  & \textbf{0.708} &         59.51  & \textbf{0.522} &         76.95  & \textbf{0.729} \\
TSIT-\method ($t=20$)  & \textbf{11.01} &         0.760  &         13.90  &         0.568  &         47.00  &         0.486  &         74.37  &         0.687  \\
TSIT-\method ($t=50$)  &         20.46  &         0.732  &         15.18  &         0.496  &         42.00  &         0.473  &         78.03  &         0.629  \\
TSIT-\method ($t=100$)  &        39.13  &         0.674  &         16.38  &         0.394  &         37.22  &         0.460  &         80.81  &         0.668  \\
TSIT-\method ($t=200$) &         80.55  &         0.518  &         18.82  &         0.249  &         36.78  &         0.446  &         88.54  &         0.629  \\
TSIT-\method ($t=400$) &        110.97  &         0.301  &        114.08  &         0.085  &         50.79  &         0.251  &        126.56  &         0.307  \\
TSIT-\method ($t=600$) &        254.44  &         0.177  &        216.62  &         0.019  &        158.69  &         0.050  &        338.89  &         0.084  \\
TSIT-\method ($t=900$) &        301.77  &         0.051  &        337.37  &         0.028  &        213.81  &         0.000  &        371.87  &         0.004  \\
TSIT-\method ($t=t^*$) &         20.46  &         0.732  & \textbf{13.03} &         0.684  & \textbf{36.78} &         0.446  & \textbf{74.37} &         0.687  \\
\bottomrule
\end{tabular}
\end{table*}

\begin{table*}[t]
\begin{minipage}[t]{0.44\textwidth}
\scriptsize
\setlength{\tabcolsep}{10pt}
\caption{
    \revise{
    \textbf{Ablation study} of $t$ near $t^*$ on AFHQ dataset.
    }
}
\label{tab:ablation_t_star_afhq}
\vspace{-10pt}
\begin{tabular}{l|ccccc}
\toprule
\revise{$t$} & \revise{0} & \revise{5 ($t^*$)} & \revise{10} & \revise{15} & \revise{25} \\
\midrule
\revise{\bf FID$\downarrow$} & \revise{13.28} & \revise{\textbf{13.03}} & \revise{13.61}  & \revise{13.92}  & \revise{14.62}  \\
\revise{\bf SSIM$\uparrow$}  & \revise{\textbf{0.708}} & \revise{0.684} & \revise{0.680}  & \revise{0.620}  & \revise{0.537}  \\
\bottomrule
\end{tabular}
\end{minipage}
\hfill
\begin{minipage}[t]{0.55\textwidth}
\scriptsize
\setlength{\tabcolsep}{8pt}
\caption{
    \revise{
    \textbf{Ablation study} of $t$ near $t^*$ on CelebA-HQ dataset.
    }
}
\label{tab:ablation_t_star_celebahq}
\vspace{-10pt}
\begin{tabular}{l|ccccccc}
\toprule
\revise{$t$}                 &   \revise{180} &   \revise{190} &       \revise{195} & \revise{200 ($t^*$)} &   \revise{205} &   \revise{210} &       \revise{220} \\
\midrule
\revise{\bf FID$\downarrow$} & \revise{37.93} & \revise{38.48} &     \revise{36.46} &       \revise{36.78} & \revise{37.09} & \revise{39.62} & \revise{\bf 36.04} \\
\revise{\bf SSIM$\uparrow$}  & \revise{0.450} & \revise{0.438} & \revise{\bf 0.455} &       \revise{0.446} & \revise{0.420} & \revise{0.418} &     \revise{0.432} \\
\bottomrule
\end{tabular}
\end{minipage}
\end{table*}

\sqq{In} the \method algorithm\sqq{,} we first gradually add noise \sqq{for both $x_0$ and $y_0$ using a shared decoder} until some preset timestep $t$.
Here, the timestep $t$ plays a critical role \sqq{in} the performance of the translator $f_{\theta}$ as well as the quality of the generated images.
As discussed in \cref{subsec:Optimal}, we proposed a simple method to determine an adequate timestep before training, denoted by $t=t^*$, by pre-computing the distance between $(x_0, x_t)$ and between $(x_t, y_t)$.
In this section, we compare the generation quality using different timesteps $t$ and show that the timestep $t^*$ selected using our method in \cref{subsec:Optimal} offers the optimal performance.

\sqq{In} \cref{fig:ablation_t_qualitative}, \sqq{we} observe that:
(1) As the translation timestep $t$ increases, the input \revise{condition} provides weaker constraint to the output generation.
For instance, the face poses of the results in row 1 and row 3 begin to change in an unwarranted way when $t>400$;
(2) When the translation timestep $t$ is small, the translation mapping can hardly approximate the real distribution (\newrevise{\textit{e.g.}}, the hair texture of the segmentation to image task in row 3, column 3).

We also present quantitative comparison results in \cref{tab:quantitative_ablation}, from which we see the trade-off between the strength of the input \revise{condition} and the difficulty \sqq{of learning} the translation mapping.
\sqq{Significantly, our} method for selecting an appropriate timestep achieves performance comparable to using the optimal $t$ shown in \cref{tab:quantitative_ablation}. This confirms the effectiveness of our simple selection strategy.

\revise{
We conduct further ablation study on the performance of timestep $t$ near the preset timestep $t^*$, \sqq{in order} to \sqq{demonstrate} the strong robustness of our strategy.
As shown in \cref{tab:ablation_t_star_afhq,tab:ablation_t_star_celebahq}, despite the significant performance drop when using different timesteps, our strategy is still able to search an adequate timestep for \method.
}

\subsection{Limitations}\label{subsec:Discussion}

Our \method method has several limitations \sqq{that are} interesting avenues for future research.
First, our algorithm is based on the assumption that both the forward and the reverse process satisfy the Markovian property, but this assumption holds only for the DDPM or its extension. 
Second, the \method is designed to train with paired data due to its reliance on using  Pix2Pix~\cite{isola2017image} or TSIT~\cite{jiang2020tsit} module as the translation mapping $f_{\theta}$.
Hence, our method cannot be applied to unpaired training data and related I2I tasks.
Third, our \method is not applicable to tasks whose condition (source domain) and the target domain are almost identical. \sqq{We briefly explain this limitation next.}
Following \cref{eq:3.3}, when $x_0$ equals $y_0$, we have $q(y_0|x_0) = \delta_{x_0}(y_0)$, which is the Dirac distribution.
Then\sqq{,} \cref{eq:3.3} becomes
\begin{align}
\mathcal L_{CE}=\log p_\theta(x_0|x_0) = 0,
\end{align}
which is a constant independent of the model parameter $\theta$. Therefore\sqq{,} the model cannot be optimized.
%
%
%


\section{Conclusion}\label{sec:conclusion}

In this paper, we propose an efficient diffusion model translator, which bridges a well-prepared DDPM and the input inference.
We provide theoretical proof to show the feasibility \sqq{of using} \sqq{this} simple module to accomplish the popular I2I task.
By using our proposed practical method to pre-select an adequate timestep and applying the forward diffusion process until this timestep, we formulate the task as the learning process of a translation mapping, without relying on any retraining of the given DDPM.
We conduct comprehensive experiments to show the high efficiency and the \sqq{outstanding performance} of our proposed algorithm.

\section*{Acknowledgements}

This work was supported by National Natural Science Foundation of China (U2336214, 62332019, 62302297), Beijing Natural Science Foundation (L222008), Shanghai Sailing Program (22YF1420300), Young Elite Scientists Sponsorship Program by CAST (2022QNRC001).

{\small
\bibliographystyle{IEEEtran}
\bibliography{ref}
}
\begin{IEEEbiography}[{\includegraphics[width=1in,height=1.25in,clip,keepaspectratio]{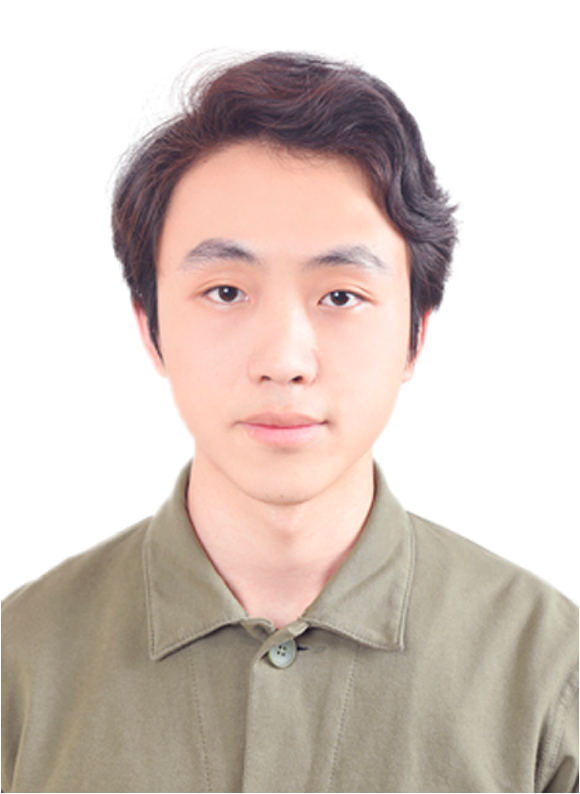}}]{Mengfei Xia}
received the B.S. degree in 2020 from the Department of Mathematical Science, Tsinghua University, Beijing, China. He is currently a fourth-year Ph.D. student at the Department of Computer Science and Technology, Tsinghua University. His
research interests include mathematical foundation in
deep learning, image processing, and computer vision. He was the recipient of the Silver Medal twice in 30th and 31st National Mathematical Olympiad of China.
\end{IEEEbiography}


\begin{IEEEbiography}[{\includegraphics[width=1in,height=1.25in,clip,keepaspectratio]{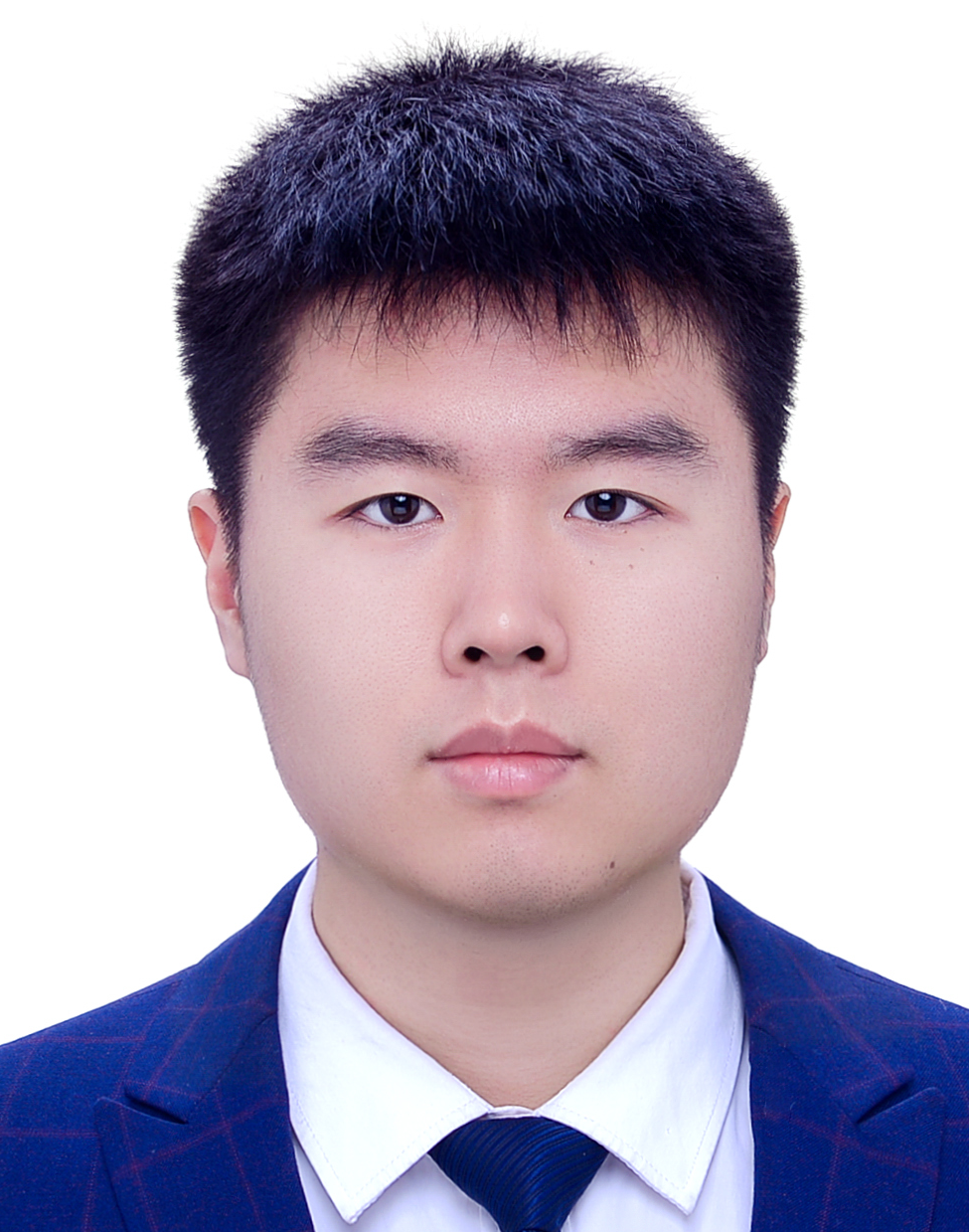}}]{Yu Zhou}
is a fourth-year undergraduate student with Zhili College, Tsinghua University, China. His research interests include deep learning and computer vision. He was the recipient of the Silver Medal twice in 35th and 36th National Olympiad in Informatics of China.
\end{IEEEbiography}


\begin{IEEEbiography}[{\includegraphics[width=1in,height=1.25in,clip,keepaspectratio]{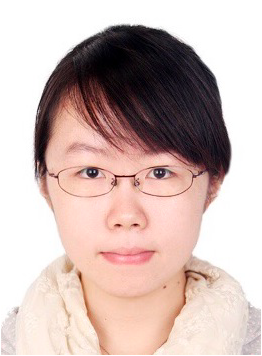}}]{Ran Yi}
is an assistant professor with the Department of Computer Science and Engineering, Shanghai Jiao Tong University. She received the BEng degree and the PhD degree from Tsinghua University, China, in 2016 and 2021. Her research interests include computer vision, computer graphics and computational geometry.
\end{IEEEbiography}


\begin{IEEEbiography}[{\includegraphics[width=1in,height=1.25in,clip,keepaspectratio]{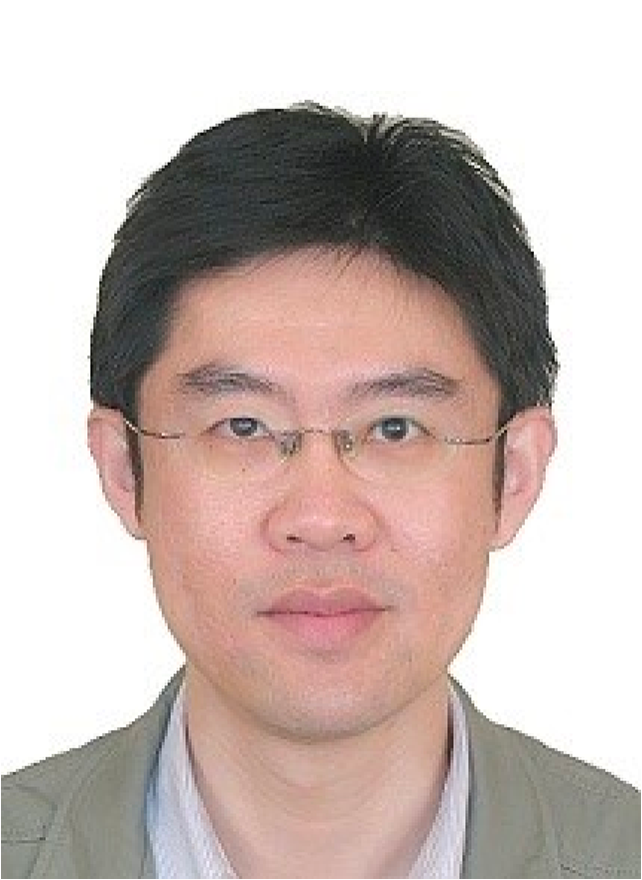}}]{Yong-Jin Liu}
is a professor with the Department of Computer Science and Technology, Tsinghua University, China. He received the BEng degree from Tianjin University, China, in 1998, and the PhD degree from the Hong Kong University of Science and Technology, Hong Kong, China, in 2004. His research interests include computer vision, computer graphics and computer-aided design. For more information, visit 
\href{http://cg.cs.tsinghua.edu.cn/people/~Yongjin/Yongjin.htm}{http://cg.cs.tsinghua.edu.cn/ people/$\sim$Yongjin/Yongjin.htm}.
\end{IEEEbiography}


\begin{IEEEbiography}[{\includegraphics[width=1in,height=1.25in,clip,keepaspectratio]{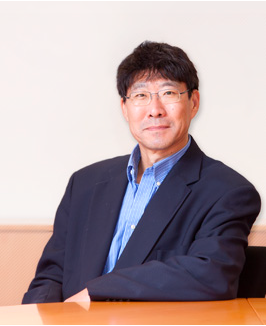}}]{Wenping Wang}
(Fellow, IEEE) received the Ph.D. degree in computer science from the University of Alberta in 1992. He is a Professor of computer science at Texas A\&M University. His research interests include computer graphics, computer visualization, computer vision, robotics, medical image processing, and geometric computing. He is or has been an journal associate editor of ACM Transactions on Graphics, IEEE Transactions on Visualization and Computer Graphics, Computer Aided Geometric Design, and Computer Graphics Forum (CGF). He has chaired a number of international conferences, including Pacific Graphics, ACM Symposium on Physical and Solid Modeling (SPM), SIGGRAPH and SIGGRAPH Asia. Prof. Wang received the John Gregory Memorial Award for his contributions to geometric modeling. He is an IEEE Fellow and an ACM Fellow.
\end{IEEEbiography}

\clearpage

\setcounter{theorem}{0}
\setcounter{lemma}{0}

\appendices

\section{Derivatives of training and inference processes}

\subsection{Algorithms of training and inference processes}

In this part, we provide the algorithms of training and inference processes.
\sqq{Notably,} both training and inference procedures in \Cref{alg:paired_train,alg:paired_infer} resemble the corresponding processes of DDPM respectively, where $\epsilon_{\phi}$ is the pre-trained DDPM with parameter $\phi$, and $\sigma_i$ is the variance of the distribution $\epsilon_{\phi}(y_{i-1}|y_i)$.
The training learns to transfer between the intermediate diffusion results $x_t$ and $y_t$ while DDPM approximator intends to predict the noise $\epsilon$ from $x_t$.

\revise{
Furthermore, we provide the algorithms \sqq{for the} training and inference process\sqq{es} of the generalized asymmetric pipelines in \Cref{alg:paired_train_asymm,alg:paired_infer_asymm}.
}

\begin{algorithm}[H]
\setstretch{1.2}
\caption{Training}\label{alg:paired_train}
\begin{algorithmic}[1]
\Repeat
\State $x_0\sim q(x_0),y_0\sim q(y_0|x_0)$
\State $z_t\sim\mathcal N(0, I)$
\State $x_t\leftarrow \sqrt{\balpha{t}}x_0+\sqrt{1-\balpha{t}}z_t$
\State $y_t\leftarrow \sqrt{\balpha{t}}y_0+\sqrt{1-\balpha{t}}z_t$
\State Take gradient descent step on
$$\nabla_{\theta}\|f_{\theta}(x_t)-y_t\|^2$$
\Until converged
\end{algorithmic}
\end{algorithm}

\begin{algorithm}[H]
\setstretch{1.2}
\caption{Inference}\label{alg:paired_infer}
\begin{algorithmic}[1]
\State $x_0\sim q(x_0)$
\State $z_t,z\sim\mathcal N(0, I)$
\State $x_t\leftarrow \sqrt{\balpha{t}}x_0+\sqrt{1-\balpha{t}}z_t$
\State $y_t\leftarrow f_{\theta}(x_t)-\sqrt{1-\balpha{t}}z_t+\sqrt{1-\balpha{t}}z$
\For {$i=t,t-1,\cdots,1$}
\State $\epsilon_i\sim\mathcal N(0, I)$ if $i>1$, else $\epsilon_i=0$
\State $y_{i-1}=\frac{(y_i-\frac{1-\alpha_i}{\sqrt{1-\bar\alpha_i}}\epsilon_{\phi}(y_i,i))}{\sqrt{\alpha_i}}+\sigma_i\epsilon_i$
\EndFor
\State \Return $y_0$
\end{algorithmic}
\end{algorithm}

\revise{
\begin{algorithm}[H]
\setstretch{1.2}
\caption{Training of the generalized asymmetric pipeline}\label{alg:paired_train_asymm}
\begin{algorithmic}[1]
\Repeat
\State $x_0\sim q(x_0),y_0\sim q(y_0|x_0)$
\State $z\sim\mathcal N(0, I)$
\State $x_s\leftarrow \sqrt{\balpha{s}}x_0+\sqrt{1-\balpha{s}}z$
\State $y_t\leftarrow \sqrt{\balpha{t}}y_0+\sqrt{1-\balpha{t}}z$
\State Take gradient descent step on
$$\nabla_{\theta}\|f_{\theta}(x_s)-y_t\|^2$$
\Until converged
\end{algorithmic}
\end{algorithm}
}

\revise{
\begin{algorithm}[H]
\setstretch{1.2}
\caption{Inference of the generalized asymmetric pipeline}\label{alg:paired_infer_asymm}
\begin{algorithmic}[1]
\State $x_0\sim q(x_0)$
\State $z_1,z_2\sim\mathcal N(0, I)$
\State $x_s\leftarrow \sqrt{\balpha{s}}x_0+\sqrt{1-\balpha{s}}z_1$
\State $y_t\leftarrow f_{\theta}(x_s)-\sqrt{1-\balpha{t}}z_1+\sqrt{1-\balpha{t}}z_2$
\For {$i=t,t-1,\cdots,1$}
\State $\epsilon_i\sim\mathcal N(0, I)$ if $i>1$, else $\epsilon_i=0$
\State $y_{i-1}=\frac{(y_i-\frac{1-\alpha_i}{\sqrt{1-\bar\alpha_i}}\epsilon_{\phi}(y_i,i))}{\sqrt{\alpha_i}}+\sigma_i\epsilon_i$
\EndFor
\State \Return $y_0$
\end{algorithmic}
\end{algorithm}
}

\begin{table*}[ht]
\begin{tabular}{cc}
\begin{minipage}[t]{0.48\textwidth}
\begin{algorithm}[H]
\setstretch{1.15}
\caption{Pseudo-code of \method in a PyTorch-like style.}
\label{alg:code}
\begin{lstlisting}[language=python]
import torch

def forward_step(x_0, y_0, t, T):
    """Defines the forward process of one training step.
    
    Args:
        x_0: Source inputs, with shape [B, C, H, W].
        y_0: Target outputs, with shape [B, C, H, W].
        t: The preset timestep to perform distribution shift.
        T: The translator module to learn.
    """
    # Compute the cumulated variance until timestep t.
    bar_alpha_t = cum_var(t)
    
    # Adding noise (i.e., diffusion) to images from both domains.
    z_t = torch.randn_like(x_0)
    x_t = torch.sqrt(bar_alpha_t) * x_0 + torch.sqrt(1 - bar_alpha_t) * z_t
    y_t = torch.sqrt(bar_alpha_t) * y_0 + torch.sqrt(1 - bar_alpha_t) * z_t
    
    # Learn the translator.
    loss = (T(x_t) - y_t).square().mean()
    
    return loss
\end{lstlisting}
\end{algorithm}
\end{minipage}
&
\begin{minipage}[t]{0.48\textwidth}
\renewcommand\arraystretch{1.0}
\begin{algorithm}[H]
\caption{\revise{Pseudo-code of \method in a PyTorch-like style.}}
\label{alg:code_asymm}
\begin{lstlisting}[language=python]
import torch

def forward_step(x_0, y_0, s, t, T):
    """Defines the forward process of one training step.
    
    Args:
        x_0: Source inputs, with shape [B, C, H, W].
        y_0: Target outputs, with shape [B, C, H, W].
        s: The preset timestep to perform distribution shift for x_0.
        t: The preset timestep to perform distribution shift for y_0.
        T: The translator module to learn.
    """
    # Compute the cumulated variance until timestep s and t.
    bar_alpha_s = cum_var(s)
    bar_alpha_t = cum_var(t)
    
    # Adding noise (i.e., diffusion) to images from both domains.
    z = torch.randn_like(x_0)
    x_s = torch.sqrt(bar_alpha_s) * x_0 + torch.sqrt(1 - bar_alpha_s) * z
    y_t = torch.sqrt(bar_alpha_t) * y_0 + torch.sqrt(1 - bar_alpha_t) * z
    
    # Learn the translator.
    loss = (T(x_s) - y_t).square().mean()
    
    return loss
\end{lstlisting}
\end{algorithm}
\end{minipage}
\end{tabular}
\end{table*}

\subsection{Pseudo-code of training process}
Our proposed diffusion model translator (\method) achieves image-to-image translation (I2I) based on a pre-trained DDPM via simply learning a distribution shift at a certain diffusion timestep.
Accordingly, it owns a \textit{highly efficient} implementation, which is even \textit{independent} of the DDPM itself.
In this part, we provide the pseudo-code of the training process in \cref{alg:code,alg:code_asymm}.
%

\section{Proofs of main results}

\begin{lemma}
We have an upper bound of the negative log-likelihood of $-\log p_{\theta}(y_0|x_0)$ by
\begin{align}
-\log p_{\theta}(y_0|x_0)\leqslant\mathbb E_q\left[\log\frac{q(y_{1:t},x_{1:t}|y_0, x_0)}{p_{\theta}(y_{0:t},x_{1:t}|x_0)}\right],
\end{align}
where $q=q(y_{1:t},x_{1:t}|y_0, x_0).$
\end{lemma}

\begin{proof}
\begin{align}
&-\log p_{\theta}(y_0|x_0) \\
\leqslant& -\log p_{\theta}(y_0|x_0) + \nonumber \\
& \qquad D_{KL}\left(q(y_{1:t},x_{1:t}|y_0, x_0) \| p_{\theta}(y_{1:t},x_{1:t}|y_0, x_0)\right) \\
=& -\log p_{\theta}(y_0|x_0)+\nonumber  \\
& \qquad \mathbb E_{q(y_{1:t},x_{1:t}|y_0, x_0)}\left[\log\frac{q(y_{1:t},x_{1:t}|y_0, x_0)}{p_{\theta}(y_{1:t},x_{1:t}|y_0, x_0)}\right] \\
=& -\log p_{\theta}(y_0|x_0)+ \nonumber \\
& \qquad \mathbb E_{q(y_{1:t},x_{1:t}|y_0, x_0)}\left[\log\frac{q(y_{1:t},x_{1:t}|y_0, x_0)}{p_{\theta}(y_{0:t},x_{1:t}|x_0)/p_{\theta}(y_0|x_0)}\right] \\
=& \mathbb E_{q(y_{1:t},x_{1:t}|y_0, x_0)}\left[\log\frac{q(y_{1:t},x_{1:t}|y_0, x_0)}{p_{\theta}(y_{0:t},x_{1:t}|x_0)}\right].
\end{align}
\end{proof}

\begin{theorem}[Closed-form expression]
The loss function in Equation (13) in the main paper has a closed-form representation.
The training is equivalent to optimizing a KL-divergence up to a non-negative constant, \textit{i.e.},
\begin{align}
\mathcal L_{VLB}=\mathbb E_{q(y_0,x_t|x_0)}\left[D_{KL}(q(y_t|y_0)\|p_{\theta}(y_t|x_t))\right] + C\sqq{.}
\end{align}
\end{theorem}

\begin{proof}
By the factorization in Equation (6) and (11) in the main paper, we observe that
\begin{align}
\mathcal L_{VLB}&=\mathbb E_{q(y_{0:t},x_{1:t}|x_0)}\left[\log\frac{q(y_{1:t},x_{1:t}|y_0, x_0)}{p_{\theta}(y_{0:t},x_{1:t}|x_0)}\right] \\
&=\mathbb E_{q(y_{0:t},x_{1:t}|x_0)}\left[\log\frac{1}{p_{\theta}(y_t|x_t)}+\sum_{j=1}^t\log\frac{q(y_j|y_{j-1})}{q(y_{j-1}|y_j)}\right].
\end{align}
Using Bayes' rule, for any $j=1,2,\cdots,t$, we have
\begin{align}
\frac{q(y_j|y_{j-1})}{q(y_{j-1}|y_j)}=\frac{q(y_j)}{q(y_{j-1})}, \quad \frac{q(y_t)}{q(y_0)}=\frac{q(y_t|y_0)}{q(y_0|y_t)}.
\end{align}

Hence\sqq{,} it is equivalent to optimizing the KL-divergence up to a non-negative constant $C$:
\begin{align}
\mathcal L_{VLB}&=\mathbb E_{q(y_{0:t},x_{1:t}|x_0)}\left[\log\frac{q(y_t|y_0)}{p_{\theta}(y_t|x_t)}+\log\frac{1}{q(y_0|y_t)}\right] \\
&=\mathbb E_{q(y_0,x_t|x_0)}\left[D_{KL}(q(y_t|y_0)\|p_{\theta}(y_t|x_t))\right] + C,
\end{align}
where $C=\mathbb E_{q(y_t)}\left[H(q(y_0|y_t))\right]\geqslant0$ and $H$ is the entropy of a distribution.
Since $q(y_t|y_0)$ follows a Gaussian distribution, then so is optimal $p_{\theta}(y_t|x_t)$.
\end{proof}

\begin{theorem}[Optimal solution to Equation (15) in the main paper]
The optimal $p_{\theta}(y_t|x_t)$ follows a Gaussian distribution with mean $\mu_{\theta}$ being
\begin{align}
\mu_{\theta}(x_t) = \sqrt{\balpha{t}}y_0.
\end{align}
\end{theorem}

\begin{proof}
To minimize the KL-divergence in Equation (15) in the main paper, we first notice that $q(y_t|y_0)$ follows a Gaussian distribution, \textit{i.e.},
\begin{align}
q(y_t|y_0)\sim\mathcal N(y_t;\sqrt{\balpha{t}}y_0,(1-\balpha{t})I),\quad \mu_t(y_t)=\sqrt{\balpha{t}}y_0,
\end{align}
which implies that $
p_{\theta}(y_t|x_t)\sim\mathcal N(y_t;\mu_{\theta}(x_t),\Sigma_{\theta}(x_t))$
with mean $\mu_{\theta}(x_t)=\mu_t(y_t)=\sqrt{\balpha{t}}y_0$.
\end{proof}

\revise{
\begin{lemma}
We have an upper bound of the negative log-likelihood of $-\log p_{\theta}(y_0|x_0)$ by
\begin{align}
-\log p_{\theta}(y_0|x_0)\leqslant\mathbb E_q\left[\log\frac{q(y_{1:t},x_{1:s}|y_0, x_0)}{p_{\theta}(y_{0:t},x_{1:s}|x_0)}\right],
\end{align}
where $q=q(y_{1:t},x_{1:s}|y_0, x_0).$
\end{lemma}
}

\revise{
\begin{proof}
\begin{align}
&-\log p_{\theta}(y_0|x_0) \\
\leqslant& -\log p_{\theta}(y_0|x_0) + \nonumber \\
& \qquad D_{KL}\left(q(y_{1:t},x_{1:s}|y_0, x_0) \| p_{\theta}(y_{1:t},x_{1:s}|y_0, x_0)\right) \\
=& -\log p_{\theta}(y_0|x_0)+\nonumber  \\
& \qquad \mathbb E_{q(y_{1:t},x_{1:s}|y_0, x_0)}\left[\log\frac{q(y_{1:t},x_{1:s}|y_0, x_0)}{p_{\theta}(y_{1:t},x_{1:s}|y_0, x_0)}\right] \\
=& -\log p_{\theta}(y_0|x_0)+ \nonumber \\
& \qquad \mathbb E_{q(y_{1:t},x_{1:s}|y_0, x_0)}\left[\log\frac{q(y_{1:t},x_{1:s}|y_0, x_0)}{p_{\theta}(y_{0:t},x_{1:s}|x_0)/p_{\theta}(y_0|x_0)}\right] \\
=& \mathbb E_{q(y_{1:t},x_{1:s}|y_0, x_0)}\left[\log\frac{q(y_{1:t},x_{1:s}|y_0, x_0)}{p_{\theta}(y_{0:t},x_{1:s}|x_0)}\right].
\end{align}
\end{proof}
}

\revise{
\begin{theorem}[Closed-form expression]
The loss function in Equation (23) in the main paper has a closed-form representation.
The training is equivalent to optimizing a KL-divergence up to a non-negative constant, \textit{i.e.},
\begin{align}
\mathcal L_{VLB}=\mathbb E_{q(y_0,x_s|x_0)}\left[D_{KL}(q(y_t|y_0)\|p_{\theta}(y_t|x_s))\right] + C
\sqq{.}
\end{align}
\end{theorem}
}

\revise{
\begin{proof}
By the factorization in Equation (20) and (21) in the main paper, we observe that
\begin{align}
\mathcal L_{VLB}&=\mathbb E_{q(y_{0:t},x_{1:s}|x_0)}\left[\log\frac{q(y_{1:t},x_{1:s}|y_0, x_0)}{p_{\theta}(y_{0:t},x_{1:s}|x_0)}\right] \\
&=\mathbb E_{q(y_{0:t},x_{1:s}|x_0)}\left[\log\frac{1}{p_{\theta}(y_t|x_s)}+\sum_{j=1}^t\log\frac{q(y_j|y_{j-1})}{q(y_{j-1}|y_j)}\right].
\end{align}
Using Bayes' rule, for any $j=1,2,\cdots,t$, we have
\begin{align}
\frac{q(y_j|y_{j-1})}{q(y_{j-1}|y_j)}=\frac{q(y_j)}{q(y_{j-1})}, \quad \frac{q(y_t)}{q(y_0)}=\frac{q(y_t|y_0)}{q(y_0|y_t)}.
\end{align}
Hence\sqq{,} it is equivalent to optimizing the KL-divergence up to a non-negative constant $C$:
\begin{align}
\mathcal L_{VLB}&=\mathbb E_{q(y_{0:t},x_{1:s}|x_0)}\left[\log\frac{q(y_t|y_0)}{p_{\theta}(y_t|x_s)}+\log\frac{1}{q(y_0|y_t)}\right] \\
&=\mathbb E_{q(y_0,x_s|x_0)}\left[D_{KL}(q(y_t|y_0)\|p_{\theta}(y_t|x_s))\right] + C,
\end{align}
where $C=\mathbb E_{q(y_t)}\left[H(q(y_0|y_t))\right]\geqslant0$ and $H$ is the entropy of a distribution.
Since $q(y_t|y_0)$ follows a Gaussian distribution, then so is the optimal $p_{\theta}(y_t|x_s)$.
\end{proof}
}

\revise{
\begin{theorem}[Optimal solution to Equation (25) in the main paper]
The optimal $p_{\theta}(y_t|x_s)$ follows a Gaussian distribution with mean $\mu_{\theta}$ being
\begin{align}
\mu_{\theta}(x_s) = \sqrt{\balpha{t}}y_0.
\end{align}
\end{theorem}
}

\revise{
\begin{proof}
To minimize the KL-divergence in Equation (25) in the main paper, we first notice that $q(y_t|y_0)$ follows a Gaussian distribution, \textit{i.e.},
\begin{align}
q(y_t|y_0)\sim\mathcal N(y_t;\sqrt{\balpha{t}}y_0,(1-\balpha{t})I),\quad \mu_t(y_t)=\sqrt{\balpha{t}}y_0,
\end{align}
which implies that $
p_{\theta}(y_t|x_s)\sim\mathcal N(y_t;\mu_{\theta}(x_s),\Sigma_{\theta}(x_s))$
with mean $\mu_{\theta}(x_s)=\mu_t(y_t)=\sqrt{\balpha{t}}y_0$.
\end{proof}
}

\begin{table*}[ht]
\begin{minipage}[t]{0.48\textwidth}
\caption{
    \newrevise{
    \textbf{Quantitative comparison} between single-step and multi-step \method (TSIT-DMT-2step) upon TSIT.
    FID and SSIM are used to evaluate the image quality and content preservation, respectively.
    }
}
\vspace{-5pt}
\centering
\SetTblrInner{rowsep=2.15pt}                
\SetTblrInner{colsep=10.0pt}               
\newrevise{
\begin{tblr}{
    cell{1-6}{1-5}={halign=c,valign=m},    
    cell{1}{2,4}={c=2}{},                  
    cell{1}{1}={r=2}{},                    
    hline{1,7}={1-5}{1.0pt},               
    hline{2}={2-5}{},                      
    hline{3}={1-5}{},                      
    vline{2,4}={1-6}{},                    
}
\label{tab:metric}
Method & TSIT-\method & & TSIT-\method-2step & \\
 & \textbf{FID$\downarrow$} & \textbf{SSIM$\uparrow$} & \textbf{FID$\downarrow$} & \textbf{SSIM$\uparrow$} \\
$t=50$ & 42.00 & 0.473  & 47.77 & 0.563 \\
$t=100$ & 37.22 & 0.460 & 43.52 & 0.563 \\
$t=200$ & 36.78 & 0.446 & 45.78 & 0.536 \\
$t=400$ & 50.79 & 0.251 & 48.34 & 0.447 \\
\end{tblr}
}
\hfill
\end{minipage}
\begin{minipage}[t]{0.48\textwidth}
\caption{
    \newrevise{
    \textbf{Time cost comparison} between single-step and multi-step \method upon TSIT.
    To measure the time cost, we report the total number of training epochs of \method (\method Epoch) and of fusion UNet (Fusion Epoch), training time for 1,000 images for \method (\method Train) and fusion UNet (Fusion Train), and inference time for a single image.
    }
}
\vspace{-10pt}
\centering
\SetTblrInner{rowsep=1.05pt}               
\SetTblrInner{colsep=10.0pt}               
\newrevise{
\begin{tblr}{
    cell{1-6}{1-3}={halign=c,valign=m},    
    hline{1,2,7}={1-3}{1.0pt},             
    hline{2}={1-3}{},                      
    vline{2,3}={1-6}{},                    
}
\label{tab:time}
Method & TSIT-\method & TSIT-\method-2step \\
DMT Train & 82s & 82s \\
DMT Epoch & 60 & 60 \\
Fusion Train & No such step & 139s \\
Fusion Epoch & No such step & 100 \\
Inference & 0.48s & 0.64s \\
\end{tblr}
}
\end{minipage}
\end{table*}

\definecolor{mygreen}{RGB}{34,170,133}
\begin{table*}[ht]
\begin{minipage}[t]{0.48\textwidth}
\centering
\vspace{-8pt}
\caption{
    \newrevise{
    \textbf{Ablation study} of $(s,t)$ pair near $s=t=200$ fixing $s=200$ on CelebA-HQ dataset.
    For clearer demonstration, original \method (\textit{i.e.}, pair with $s=t$) is highlighted in \textbf{\textcolor{gray}{gray}}.
    }
}
\label{tab:ablation_s_t_sstar}
\vspace{-10pt}
\SetTblrInner{rowsep=2.15pt}                
\SetTblrInner{colsep=10.0pt}               
\newrevise{
\begin{tblr}{
    cell{1-18}{1-5}={halign=c,valign=m},    
    cell{2}{1}={r=17}{},                    
    hline{1,19}={1-5}{1.0pt},               
    hline{2}={1-5}{},                      
    vline{2}={2-18}{},                    
    vline{3}={1-18}{},                    
    cell{10}{2-5}={bg=lightgray!35},
}
$s$ & $t$ & \textbf{FID$\downarrow$} & \textbf{SSIM$\uparrow$} & \textbf{LPIPS$\downarrow$} \\
$s=200$ & $t=0$ & 53.61 & 0.347 & 0.492 \\
 & $t=50$ & 37.27 & 0.445 & 0.443 \\
 & $t=100$ & 37.31 & 0.447 & 0.443 \\
 & $t=150$ & 43.43 & 0.441 & 0.445 \\
 & $t=160$ & 43.35 & 0.438 & 0.449 \\
 & $t=170$ & 44.69 & 0.458 & 0.428 \\
 & $t=180$ & 42.82 & \bf 0.466 & \bf 0.428 \\
 & $t=190$ & 45.94 & 0.458 & 0.430 \\
 & $t=200$ & \bf 36.78 & 0.446 & 0.433 \\
 & $t=210$ & 45.15 & 0.425 & 0.433 \\
 & $t=220$ & 46.17 & 0.422 & 0.436 \\
 & $t=230$ & 46.78 & 0.418 & 0.450 \\
 & $t=240$ & 47.69 & 0.415 & 0.451 \\
 & $t=250$ & 48.32 & 0.377 & 0.464 \\
 & $t=300$ & 55.54 & 0.337 & 0.500 \\
 & $t=350$ & 64.66 & 0.267 & 0.567 \\
 & $t=400$ & 53.42 & 0.328 & 0.493 \\
\end{tblr}
}
\end{minipage}
\hfill
\begin{minipage}[t]{0.48\textwidth}
\centering
\vspace{-8pt}
\caption{
    \newrevise{
    \textbf{Ablation study} of $(s,t)$ pair near $s=t=200$ fixing $t=200$ on CelebA-HQ dataset.
    For clearer demonstration, original \method (\textit{i.e.}, pair with $s=t$) is highlighted in \textbf{\textcolor{gray}{gray}}.
    }
}
\label{tab:ablation_s_t_tstar}
\vspace{-10pt}
\SetTblrInner{rowsep=2.15pt}                
\SetTblrInner{colsep=10.0pt}               
\newrevise{
\begin{tblr}{
    cell{1-18}{1-5}={halign=c,valign=m},    
    cell{2}{2}={r=17}{},                    
    hline{1,19}={1-5}{1.0pt},               
    hline{2}={1-5}{},                      
    vline{2}={2-18}{},                    
    vline{3}={1-18}{},                    
    cell{10}{2-5}={bg=lightgray!35},
    cell{10}{1}={bg=lightgray!35},
}
$s$ & $t$ & \textbf{FID$\downarrow$} & \textbf{SSIM$\uparrow$} & \textbf{LPIPS$\downarrow$} \\
$s=0$ & $t=200$ & 44.98 & 0.448 & 0.453 \\
$s=50$ & & 42.93 & 0.451 & 0.437 \\
$s=100$ & & 44.50 & 0.443 & 0.457 \\
$s=150$ & & 43.60 & 0.446 & 0.437 \\
$s=160$ & & 42.72 & 0.445 & 0.451 \\
$s=170$ & & 44.37 & 0.443 & 0.427 \\
$s=180$ & & 44.35 & 0.447 & 0.432 \\
$s=190$ & & 45.49 & 0.438 & \bf 0.427 \\
$s=200$ & & \bf 36.78 & 0.446 & 0.433 \\
$s=210$ & & 43.99 & 0.441 & 0.442 \\
$s=220$ & & 44.94 & 0.451 & 0.432 \\
$s=230$ & & 45.67 & 0.444 & 0.439 \\
$s=240$ & & 45.68 & 0.441 & 0.441 \\
$s=250$ & & 44.11 & 0.437 & 0.433 \\
$s=300$ & & 45.36 & \bf 0.453 & 0.432 \\
$s=350$ & & 45.39 & 0.451 & 0.421 \\
$s=400$ & & 44.44 & 0.445 & 0.437 \\
\end{tblr}
}
\end{minipage}
\end{table*}

\section{\newrevise{Comparison between multi-step and asymmetric \method}}

\newrevise{
\yr{\textbf{Multi-step \method.}} To implement the multi-step \method, due to the use of the vanilla DDPM, which is only capable of inputting a 3-channel input intermediate noisy image, we train an auxiliary UNet model to fuse the $y_{t/2}$ transformed from $x_{t/2}$ together with the $y'_{t/2}$ denoised from the $y_{t}$.
}

\newrevise{
In order to train the fusion UNet processing the intermediate noisy images $y_{t/2}$ and $y'_{t/2}$, we first need to prepare the dataset.
In more details, given a paired data $(x_0,y_0)$, the preset timestep $t$, noise $z$ under standard Gaussian distribution, the pre-trained \method $G$, and the pre-trained diffusion model, we first apply the diffusion forward process onto both $x_0$ and $y_0$ with noise $z$ until timestep $t$ and $t/2$, \textit{i.e.}, we acquire the $x_{t/2}$, $x_{t}$, $y_{t/2}$, and $y_{t}$.
Next, we utilize the pre-trained \method model to obtain the transformed $G(x_{t/2})$ and $G(x_{t})$.
Then, we apply the reverse process via the pre-trained diffusion model to achieve the denoised result from $G(x_{t})$, denoted by $D(G(x_{t}))$.
Finally, repeating the process above with different paired $(x_0,y_0,z)$, we are able to acquire the dataset for the fusion UNet.
Note that similar to the training of \method, the fusion UNet are aimed to deal with \textit{noisy} images.
That is to say, empirically we need much more data samples and training epochs since CNN may easily fail on noisy data.
For the CelebA-HQ dataset with 30,000 images, we train the fusion UNet with more than 200,000 samples.
As a comparison, we train \method within only 60 epochs with 27,000 data samples.
}

\newrevise{
The comparisons between multi-step and our single-step design are reported in \cref{tab:metric} and \cref{tab:time}.
From \cref{tab:metric} we observe that under the same timestep, the FID score of multi-step \method is worse than single-step in most cases, and multi-step \method does not significantly benefit from the additional step of performing the fusion UNet.
As for SSIM score, there is indeed performance improvement to some extent compared to the vanilla DMT, which is mainly due to doubling the information during denoising process and taking advantage of the fusion UNet.
Considering the dramatic additional time cost discussed below, we regard that the vanilla single-step DMT is an adequate solution to the I2I task. 
It is noteworthy that for multi-step \method, both the training (including \method and Fusion training) and inference costs increase significantly, as shown in \cref{tab:time}, which confirms the superiority of the single-step \method proposed in the paper.
}

\newrevise{
\yr{\textbf{Asymmetric \method.}}
As for asymmetric \method, in addition to the theoretical analysis in the main paper, we also conducted a comprehensive ablation study focusing on the performance at various timestep pairs $(s,t)$ near $s=t$.
As shown in \cref{tab:ablation_s_t_sstar,tab:ablation_s_t_tstar}, our proposed strategy (pair with $s=t$) is capable of achieving on-par or even superior performance across various $(s,t)$ alternatives.
}

\section{More results}

This part shows more qualitative results and compares our \method with existing I2I approaches, including Pix2Pix (GAN-based)~\cite{isola2017image}, TSIT (GAN-based)~\cite{jiang2020tsit} and Palette (DDPM-based)~\cite{saharia2021palette}.
We perform evaluation on the tasks of stylization (\cref{fig:portrait_qmupd0-1-0}), image colorization (\cref{fig:afhq}), segmentation to image (\cref{fig:celeba}), and sketch to image (\cref{fig:edges2handbags}), using our handcrafted Anime dataset, AFHQ~\cite{choi2020stargan}, CelebA-HQ~\cite{karras2018progressive}, and Edges2handbags~\cite{zhu2016generative,xie15hed}, respectively.
Our method surpasses the other three competitors with higher fidelity (\textit{e.g.}, clearer contours\newrevise{, less artifacts} and more realistic colors \newrevise{as highlighted}), suggesting that our \method manages to bridge the content information provided by the input \revise{condition} and the domain knowledge contained in the pre-trained DDPM.

\begin{figure*}[!ht]
\centering
\includegraphics[width=1.0\textwidth]{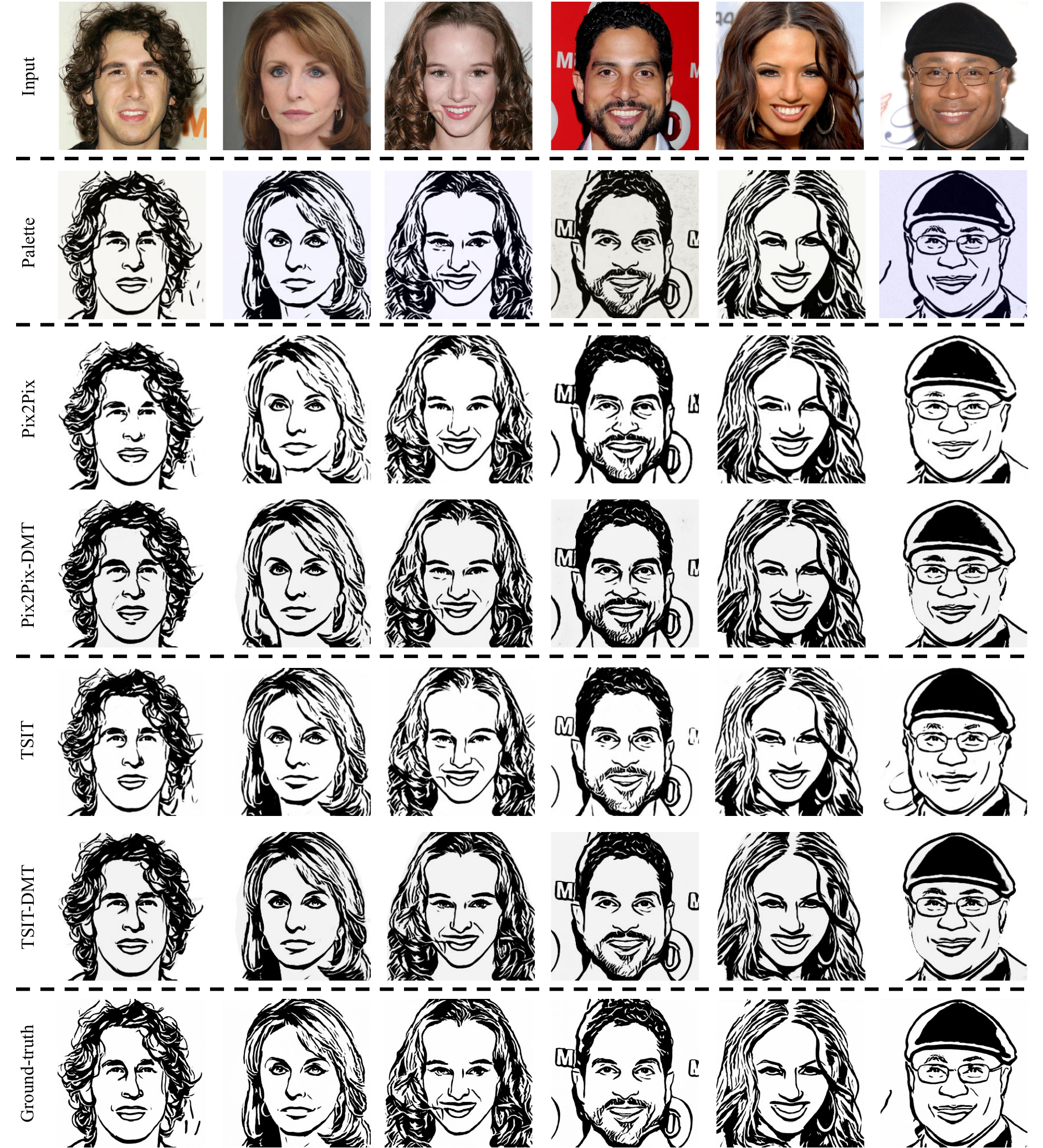}
\vspace{-15pt}
\caption{
    \textbf{Qualitative comparison} when translating human face images to portraits, using our handcrafted Portrait dataset.
}
\label{fig:portrait_qmupd0-1-0}
\vspace{-5pt}
\end{figure*}
\begin{figure*}[!ht]
\centering
\includegraphics[width=1.0\textwidth]{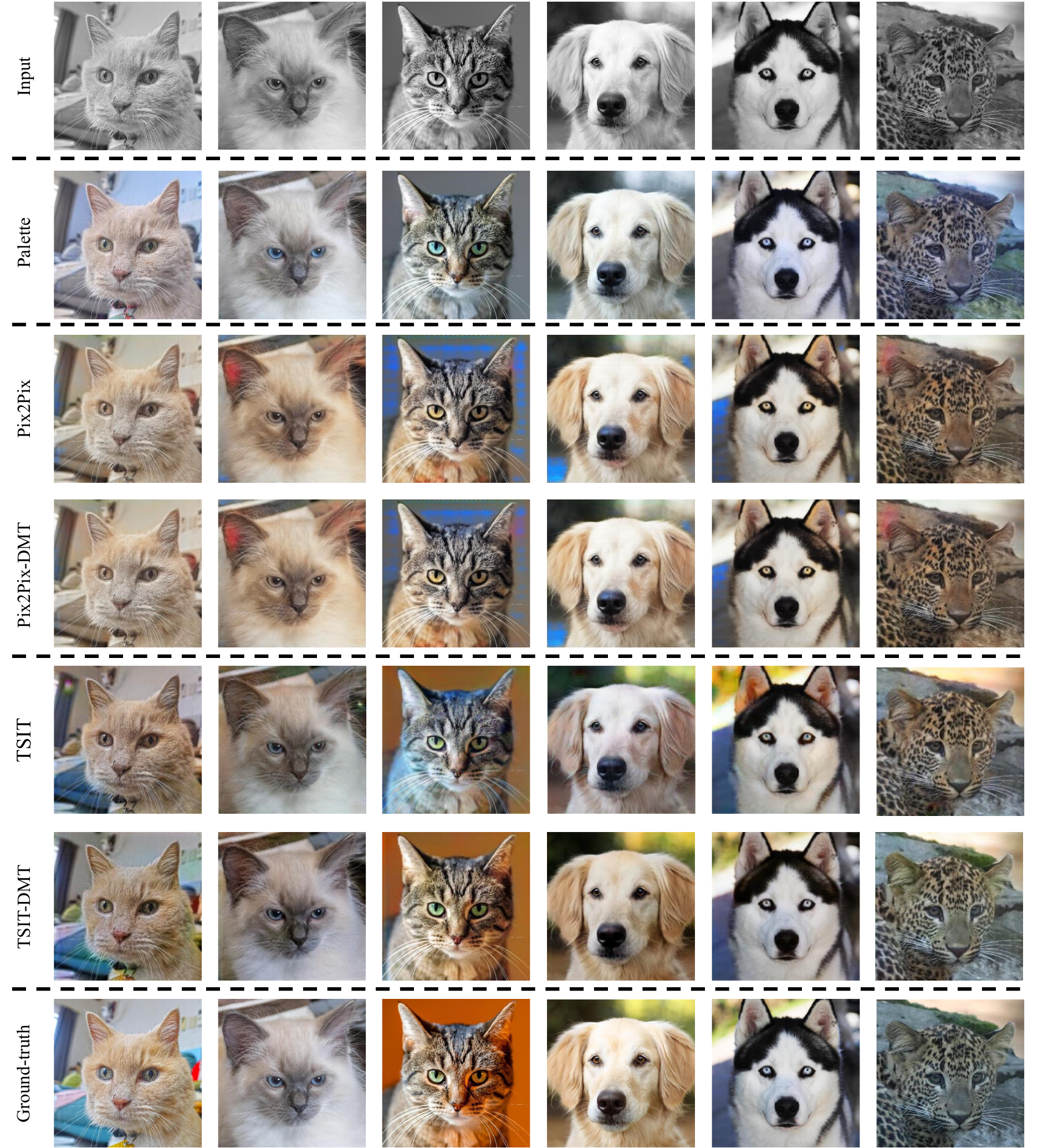}
\vspace{-15pt}
\caption{
    \textbf{Qualitative comparison} when translating greyscale images to colorized ones, using AFHQ dataset~\cite{choi2020stargan}.
}
\label{fig:afhq}
\vspace{-5pt}
\end{figure*}
\begin{figure*}[!ht]
\centering
\includegraphics[width=1.0\textwidth]{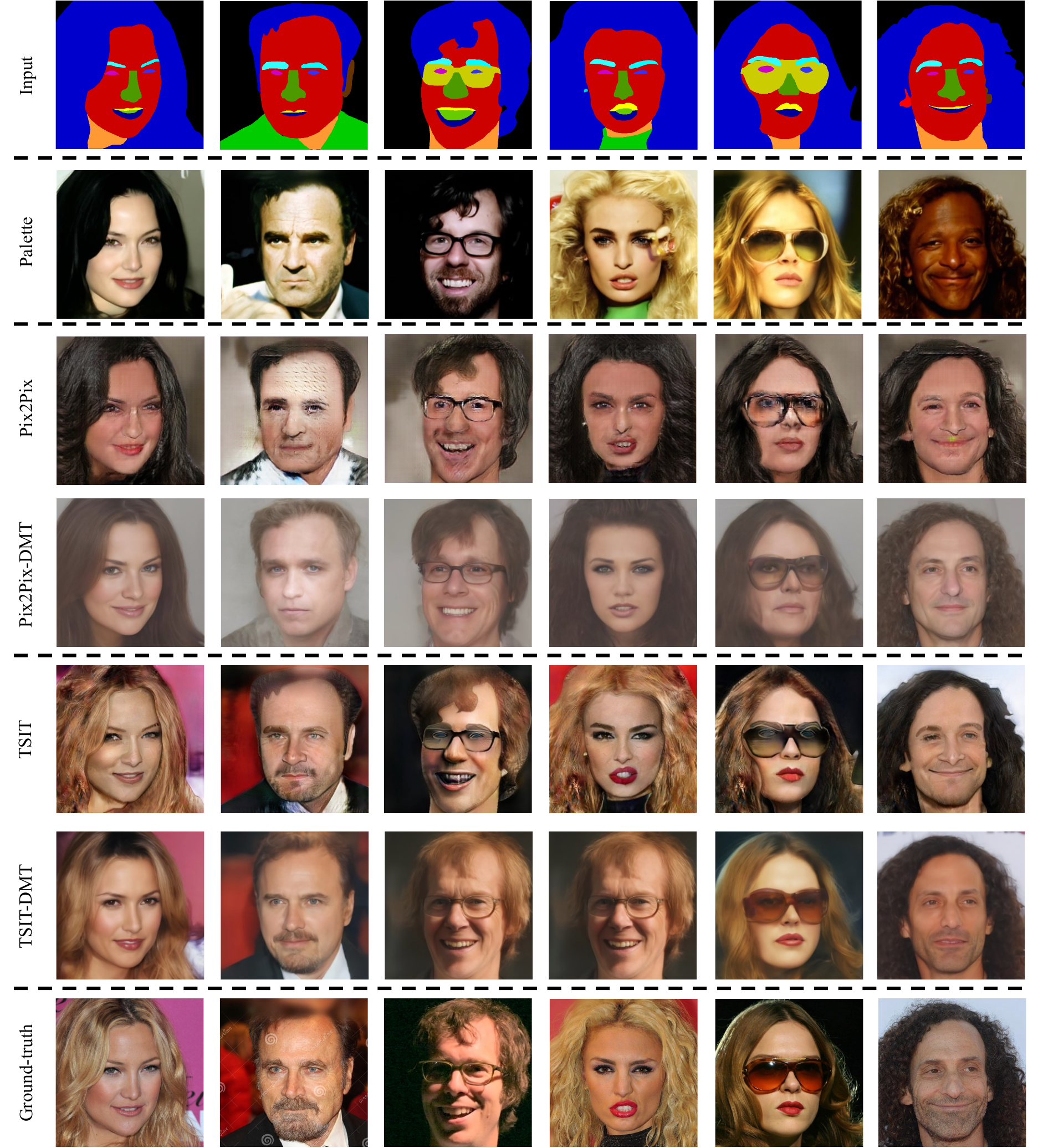}
\vspace{-15pt}
\caption{
    \textbf{Qualitative comparison} when translating segmentation maps to images, using CelebA-HQ dataset~\cite{karras2018progressive}.
}
\label{fig:celeba}
\vspace{-5pt}
\end{figure*}
\begin{figure*}[!ht]
\centering
\includegraphics[width=1.0\textwidth]{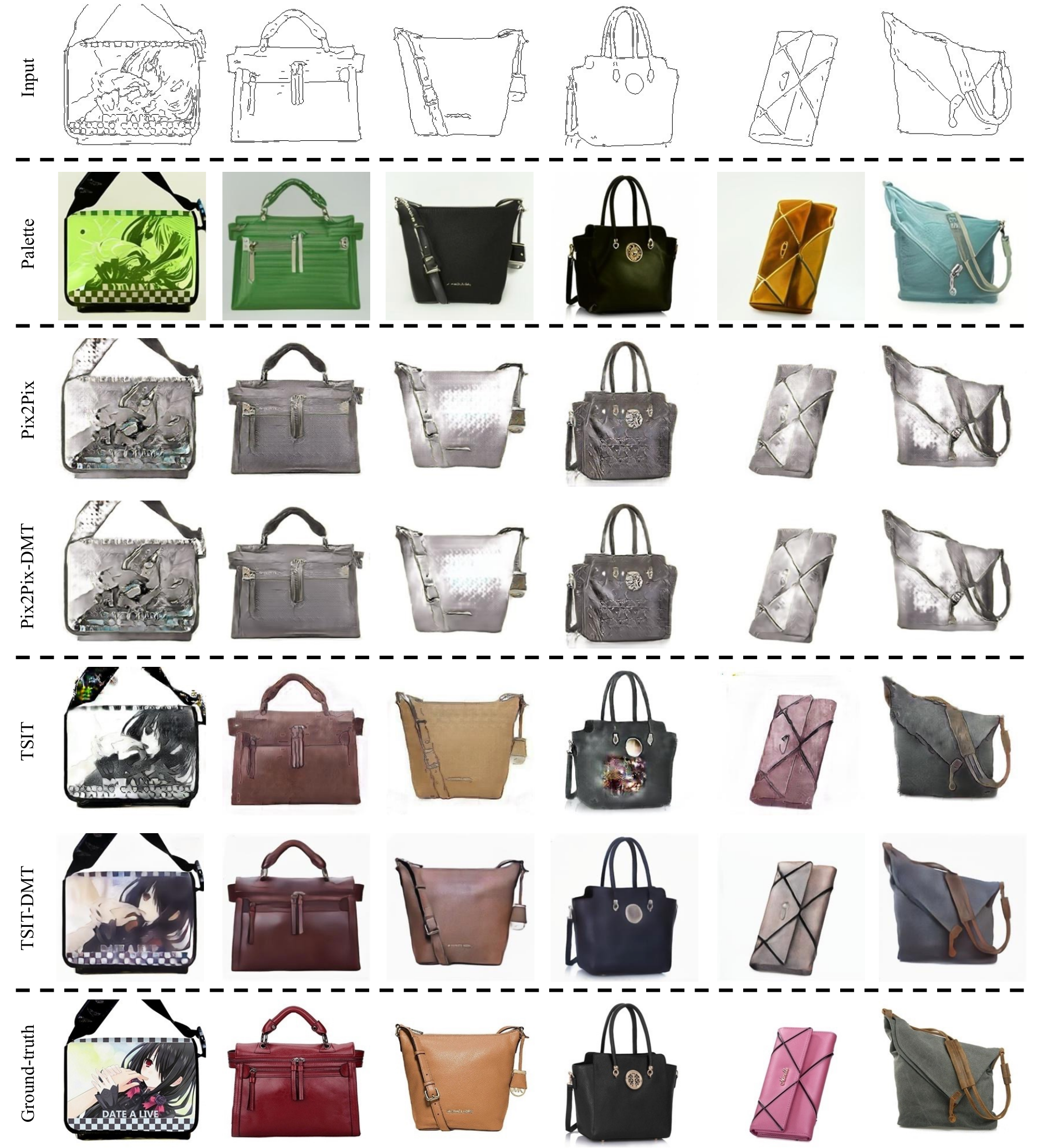}
\vspace{-15pt}
\caption{
    \textbf{Qualitative comparison} when translating sketches to images, using Edges2handbags dataset~\cite{zhu2016generative,xie15hed}.
}
\label{fig:edges2handbags}
\vspace{-5pt}
\end{figure*}

\end{document}